\documentclass{article} % For LaTeX2e
\usepackage{iclr2027_conference,times}

\usepackage{hyperref} 
\usepackage{url}      
\usepackage{booktabs} 
\usepackage{amsfonts} 
\usepackage{nicefrac} 
\usepackage{microtype}
\usepackage[dvipsnames]{xcolor}
\usepackage{colortbl}

\usepackage{graphicx}
\usepackage{subcaption}
\usepackage{wrapfig}
\usepackage{enumitem}

\usepackage{amsmath}
\usepackage{amssymb}
\usepackage{mathtools}
\usepackage{amsthm}
\usepackage{float}

\newcommand{\E}{\mathbb{E}}
\newcommand{\R}{\mathbb{R}}
\newcommand{\vect}[1]{\mathbf{#1}}

\usepackage{algorithm}
\usepackage{algorithmic}

\usepackage[capitalize,noabbrev]{cleveref}
\usepackage{etoc} % appendix-only table of contents

\theoremstyle{plain}
\newtheorem{theorem}{Theorem}[section]

\newtheorem{lemma}[theorem]{Lemma}
\newtheorem{corollary}[theorem]{Corollary}
\theoremstyle{definition}
\newtheorem{definition}[theorem]{Definition}
\newtheorem{assumption}[theorem]{Assumption}
\theoremstyle{remark}
\newtheorem{remark}[theorem]{Remark}
\crefname{assumption}{Assumption}{Assumptions}
\Crefname{assumption}{Assumption}{Assumptions}

\newcommand{\method}{\text{COVER}}

\usepackage{multirow}
\usepackage[normalem]{ulem}
\newcommand{\sd}[1]{\,{\scriptsize\textcolor{black!55}{$\pm#1$}}}
\definecolor{figred}{RGB}{182,67,66}
\definecolor{figgreen}{RGB}{46,158,68}
\title{Uncovering Cross-Objective Interference in Multi-Objective Alignment}

\author{%
  Yining Lu \\
  University of Notre Dame \\
  \texttt{ylu33@nd.edu}
  \And Meng Jiang \\
  University of Notre Dame\\
  \texttt{mjiang2@nd.edu}
}

\usepackage{placeins}

\newif\ifarxiv
\arxivtrue
\ifarxiv \iclrfinalcopy \fi

\begin{document}

\maketitle
% must follow \maketitle, which sets the ICLR header
\ifarxiv \lhead{Uncovering Cross-Objective Interference in Multi-Objective Alignment} \fi
\etocdepthtag.toc{mainmatter}

\begin{abstract}
  We study a persistent failure mode in multi-objective alignment for large language models (LLMs), in which scalarized training improves only some objectives while the others degrade. We formalize this phenomenon as cross-objective interference and, to our knowledge, conduct the first systematic study of scalarization algorithms for multi-objective LLM alignment. The study shows that interference is pervasive across algorithms yet strongly model-dependent. To understand how interference arises, we derive a local covariance law stating that an objective improves or degrades at first order according to the sign of the covariance between its reward and the scalarized score. We extend this law to the clipped surrogate objectives of modern reinforcement fine-tuning and show that it still holds under mild conditions. Building on this law, we propose COVariance-floor Enforced Reweighting (\method{}), a one-sided controller that raises an objective's weight only when the covariance between its reward and the clipped advantage weight falls below a target. Through extensive experiments, we find that \method{} can mitigate cross-objective interference while matching linear scalarization when objectives already co-improve. Finally, to explain why interference is model-dependent, we complement the local covariance law with a global convergence analysis. This analysis gives sufficient conditions for the non-convex scalarized objective to satisfy the Polyak--\L{}ojasiewicz condition and relates interference to model geometry.\footnote{Code to reproduce our results is available at \href{https://github.com/yining610/COVER}{https://github.com/yining610/COVER}.}
\end{abstract}

\section{Introduction}

\begin{figure*}[ht]
  \centering
  \setlength{\tabcolsep}{1pt}
  \renewcommand{\arraystretch}{0}
  \begin{subfigure}{\textwidth}
    \centering
    \begin{tabular}{ccc}
      \includegraphics[width=0.33\textwidth]{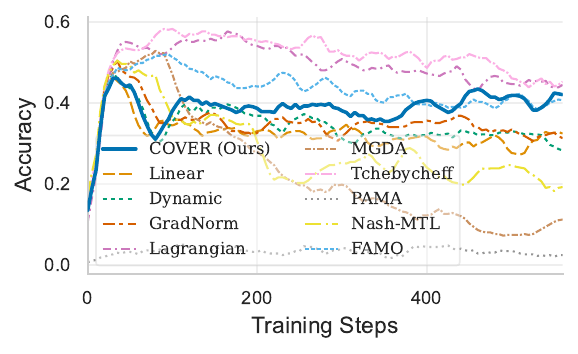} &
      \includegraphics[width=0.33\textwidth]{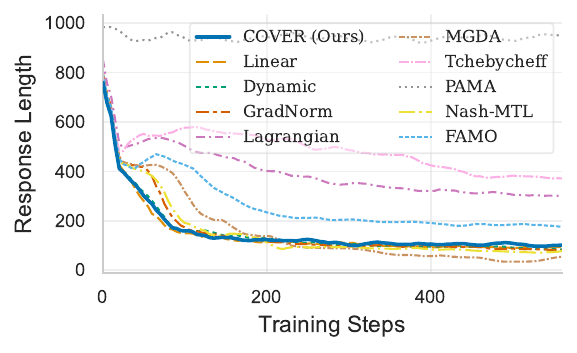} &
      \includegraphics[width=0.33\textwidth]{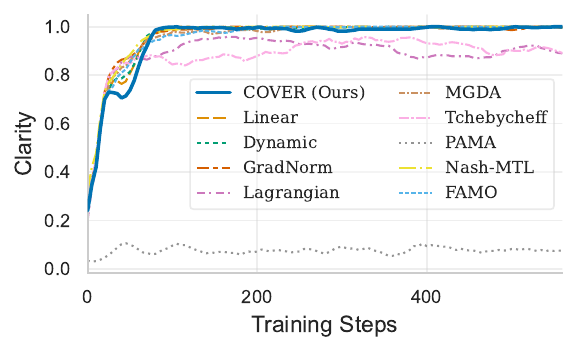}
    \end{tabular}
    \caption{Qwen2.5-1.5B-Base}
    \label{fig:intro-qwen2.5-base}
  \end{subfigure}

  \vspace{-2pt}

  \begin{subfigure}{\textwidth}
    \centering
    \begin{tabular}{ccc}
      \includegraphics[width=0.33\textwidth]{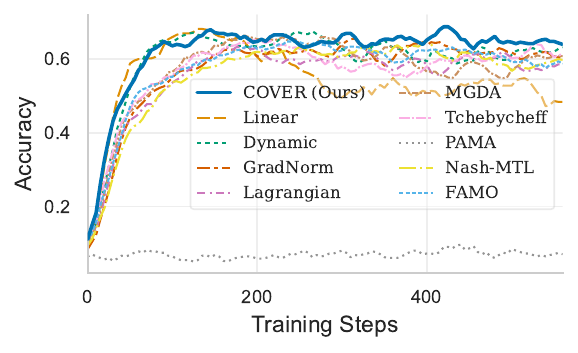} &
      \includegraphics[width=0.33\textwidth]{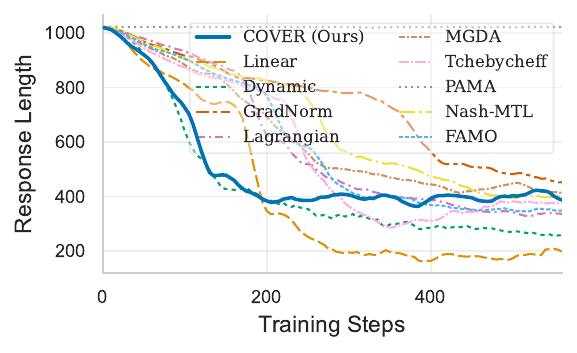} &
      \includegraphics[width=0.33\textwidth]{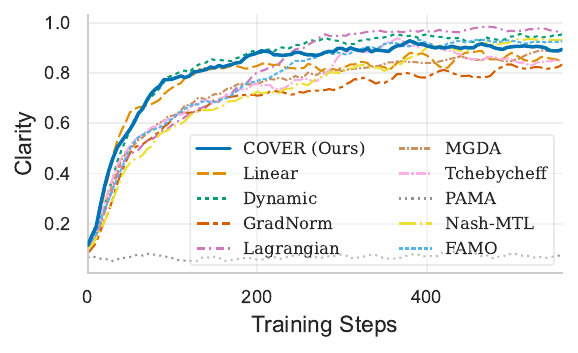}
    \end{tabular}
    \caption{Qwen3-1.7B-Base}
    \label{fig:reinforce-qwen3-base}
  \end{subfigure}
  \caption{\textbf{MATH-500 test performance over training on Qwen2.5-1.5B-Base and Qwen3-1.7B-Base.} Accuracy, response length in tokens (lower is more concise), and clarity under REINFORCE for nine scalarization algorithms and \method{}. Curves are smoothed with a centered moving average over five evaluations. On Qwen2.5-1.5B-Base, every method except \method{} (ours) and PAMA (which fails to learn) improves conciseness and clarity at the expense of accuracy, a clear instance of cross-objective interference. Under the same training procedure, eight of ten methods co-improve all three objectives on Qwen3-1.7B-Base, only Linear interferes, and PAMA again fails to learn.}
  \label{fig:main intro figure}
  \vspace{-4pt}
\end{figure*}

A deployed LLM is expected to satisfy several objectives at once, for example to answer correctly while staying concise, clear, and safe. Multi-objective LLM alignment trains a model on such objectives jointly, predominantly through reinforcement fine-tuning (RFT) \citep{ouyang2022traininglanguagemodelsfollow, bai2022constitutionalaiharmlessnessai, lambert2025tulu3pushingfrontiers, shen2025simultaneousmultiobjectivealignmentverifiable}. The common practice is \textit{scalarization}, which reduces the problem to a single scalar objective by combining the objectives with static \citep{kimiteam2025kimik15scalingreinforcement} or dynamic \citep{lu2026learningoptimizemultiobjectivealignment} weights, either at the reward level \citep{guo-etal-2024-controllable} or at the gradient level \citep{li-etal-2025-gradient}. Despite its popularity, scalarization has a persistent and underexplored failure mode. Training often keeps improving the objectives that are easiest to optimize while the others degrade, a pattern we formalize as \textit{cross-objective interference}.

To test whether this failure is specific to simple weighting schemes or inherent to scalarization, we evaluate nine well-established algorithms from multi-task learning (MTL) and multi-objective optimization (MOO), five at the reward level and four at the gradient level. The reward-level methods include two classical ones, linear weighting (Linear, \citep{10.1145/1390156.1390162}) and the Lagrangian primal-dual method (Lagrangian, \citep{NIPS2013_3493894f}), and three more recent ones, smooth Tchebycheff scalarization (Tchebycheff hereafter, \citep{lin2024smoothtchebycheffscalarizationmultiobjective, li2026multiobjective}), PAMA \citep{10.1007/978-3-032-06078-5_15}, and dynamic reward weighting (Dynamic, \citep{lu2026learningoptimizemultiobjectivealignment}). The gradient-level methods are MGDA \citep{mgda}, GradNorm \citep{pmlr-v80-chen18a}, Nash-MTL \citep{pmlr-v162-navon22a}, and FAMO \citep{NEURIPS2023_b2fe1ee8}.
To our knowledge, this is the first systematic evaluation of scalarization algorithms for multi-objective LLM alignment.

Our results show that interference is pervasive across algorithms yet strongly model-dependent, suggesting that it runs deeper than scalarization itself. On Qwen2.5-1.5B-Base (\Cref{fig:intro-qwen2.5-base}) and Qwen2.5-1.5B-IFT (Appendix~\Cref{fig:reinforce-ift-appendix}), accuracy falls from its peak under nearly every algorithm while conciseness and clarity improve. Under the same training procedure, however, only Linear interferes on Qwen3-1.7B-Base (\Cref{fig:reinforce-qwen3-base}). At 7B and 8B parameters, accuracy again falls on every model except Qwen3-8B-Base (Appendix~\Cref{fig:large models}). Moreover, interference arises even when the objectives do not strongly conflict under the gradient-based definitions commonly used in MTL and MOO \citep{10.1145/1014052.1014067, NEURIPS2020_3fe78a8a, NEURIPS2021_9d27fdf2, shi2023recon, kim2025conflictaverse}.\footnote{See Appendix~\ref{appendix:experiments}, \Cref{fig:reinforce-ift-appendix,fig:smollm-reinforce,fig:large models} for experiments across models, and \Cref{fig:cosine similarity} for gradient analysis.} Existing theories on linear scalarization \citep{lu2023multiobjective}, convexity \citep{wei2021fairnessaccuracypareto}, gradient conflict \citep{Sener2018MultiTaskLA}, and generalization tradeoffs \citep{NEURIPS2023_ddcf3462} leave these findings unexplained, because they were developed for supervised learning or rest on assumptions, such as convexity, that do not transfer cleanly to LLM alignment.

To address this gap, we develop a theoretical framework analyzing multi-objective alignment through first-order improvement conditions. Starting from the classic policy gradient \citep{NIPS1999_464d828b, hu2025reinforcestabilizingcriticfreepolicy}, we derive a reward-level local covariance law, which states that an objective improves at first order under scalarized RFT \textit{when its reward has positive covariance with the scalarized score}. The law suggests how cross-objective interference can arise. Objectives that are easy to optimize can dominate the scalarized score and leave a harder objective with small or negative covariance. That objective then stalls or degrades even as the expected scalarized score increases. On IFEval, for example, helpfulness has the smallest covariance with the scalarized score and does not improve, whereas constraint following has a far larger covariance and improves (\Cref{tab:cov-per-method}).

We then extend this analysis to the clipped surrogate objectives used in modern RFT, such as GRPO \citep{shao2024deepseekmathpushinglimitsmathematical}, and show that the covariance law still holds under mild conditions. This directly motivates our method, COVariance-floor Enforced Reweighting (\method{}), a one-sided controller that raises an objective's weight only when the covariance between its reward and the clipped advantage weight falls below a target. When every covariance reaches its target, \method{} leaves the weights unchanged and trains like linear scalarization. 
Empirically, \method{} can reduce the drop in an objective that declines gradually. On MATH-500, \method{} lowers the accuracy drop (peak minus final) from 0.149 under Linear to 0.101 and raises final accuracy from 0.358 to 0.435, and on MBPP, \method{} cuts the pass-rate drop from 0.141 to 0.090. When objectives already co-improve, as on GSM8K, \method{} ends at most 0.02 below Linear on every objective.

Because the covariance law is local, it predicts each objective's next change but not where training converges. It therefore cannot explain why, under the same training procedure, most methods interfere on the Qwen2.5 models but co-improve all three objectives on Qwen3-1.7B-Base. To explain this model dependence, we study the global geometry of scalarized RFT using the Polyak--\L{}ojasiewicz (PL) condition, which accommodates non-convex objectives. 
We derive sufficient conditions under which the scalarized RFT objective satisfies the $\mu$-PL condition. The resulting PL constant $\mu$ suggests a second, model-dependent interference mechanism, in which optimization can stall near suboptimal solutions that favor easy objectives. The constant is small when (i) the policy assigns little probability mass to the optimal completion, (ii) the optimal completion has a small scalarized score, or (iii) token-level gradient contributions cancel under an ill-conditioned logit Jacobian.
Together, these analyses suggest that cross-objective interference is both algorithmic (covariance misalignment), which \method{} acts on, and model-dependent (unfavorable geometry). Both findings provide actionable insights for robust multi-objective alignment. In summary, our contributions are fourfold:
\begin{itemize}[leftmargin=*]
\item \textbf{Systematic empirical study:} We provide, to our knowledge, the first systematic evaluation of MTL and MOO scalarization algorithms for LLM alignment, finding that cross-objective interference is pervasive yet strongly model-dependent (\S\ref{sec:experiments}).

\item \textbf{Local improvement theory:} We derive a reward-level local covariance law giving first-order conditions for objective improvement (\S\ref{subsec:reinforce-kl-cov}) and extend it to clipped surrogate objectives (\S\ref{subsec:grpo-clipping}).

\item \textbf{Covariance controller:} We propose \method{}, a one-sided controller that raises an objective's scalarization weight only when the covariance between its reward and the clipped advantage weight falls below a target. \method{} can reduce the drop in an objective that declines gradually and stays close to Linear when objectives co-improve (\S\ref{sec:ctwa}).

\item \textbf{Global convergence analysis:} We derive sufficient conditions for scalarized RFT to satisfy the PL condition and relate cross-objective interference to model geometry (\S\ref{sec:pl}).
\end{itemize}

\section{Related Work}
\subsection{Multi-Task Learning: Gradient Conflicts and Solutions}
MTL trains one model on several tasks jointly, and negative transfer between tasks commonly arises from conflicting gradients and imbalanced loss scales.
Common solutions either reweight tasks, as in GradNorm \citep{pmlr-v80-chen18a}, Nash-MTL \citep{pmlr-v162-navon22a}, and FAMO \citep{NEURIPS2023_b2fe1ee8}, or modify gradients directly, as in PCGrad \citep{NEURIPS2020_3fe78a8a}, CAGrad \citep{NEURIPS2021_9d27fdf2}, Gradient Vaccine \citep{wang2021gradient}, Recon \citep{shi2023recon}, and DB-MTL \citep{lin2025dualbalancingmultitasklearning}. While these MTL methods provide valuable insights for gradient control or loss balance, they were developed primarily for supervised settings with convex or well-behaved loss structures, and their behavior in LLM RFT is largely untested.

\subsection{Multi-Objective Optimization: Scalarization and Pareto Optimality}
Classic MOO reduces multi-objective problems to single-objective optimization via scalarization, with convergence guarantees typically relying on convexity or determinism assumptions that do not transfer cleanly to LLM alignment. 
The foundational gradient-based approach MGDA \citep{mgda} computes a common descent direction by solving a minimum-norm problem over the convex hull of objective gradients. Recent extensions include PMGDA \citep{zhang2024pmgdapreferencebasedmultiplegradient}, which incorporates user preferences, and PAMA \citep{10.1007/978-3-032-06078-5_15}, which solves a reward-level minimum-norm problem for LLM alignment.
While \citet{lu2023multiobjective} analyze when linear scalarization can in principle recover the full (non-convex) Pareto front, their guarantees rely on strong non-determinism and numerical stability assumptions. Stochastic MOO methods such as SMG \citep{liu2021stochasticmultigradientalgorithmmultiobjective}, MoCo \citep{fernando2023mitigating}, and SDMGrad \citep{NEURIPS2023_0e5b96f9} relax deterministic assumptions but remain primarily designed for supervised settings. 

Recent work has refined Pareto stationarity concepts \citep{hu2025leveraging} and explored regimes where multiple objectives can facilitate optimization \citep{pmlr-v202-dann23a, pmlr-v267-efroni25a}. In multi-objective RL, CA-NPG \citep{10.1109/TPAMI.2025.3528944} and conflict-averse updates \citep{kim2025conflictaverse} aim to improve all objectives under KL or safety constraints, but they do not address the covariance misalignment we identify. None of these works formalizes cross-objective interference or explains its dependence on model geometry.

\subsection{Multi-Objective LLM Alignment}
Most multi-objective alignment studies reduce multiple rewards to a scalar objective through reward-level scalarization, including static linear scalarization \citep{10.5555/3666122.3668696, zhang-zuo-2025-grpo, yao2025training}, dynamic reward weighting \citep{lu2026learningoptimizemultiobjectivealignment}, Lagrangian relaxation \citep{moskovitz2024confronting}, Tchebycheff scalarization and its variants \citep{steuer1983interactive, lin2024smoothtchebycheffscalarizationmultiobjective, li2026multiobjective}, and GDPO \citep{liu2026gdpo}, which normalizes the advantage of each reward separately and then sums these advantages. These methods can be extended to produce steerable policies that adapt to user preferences \citep{basaklar2023pdmorl, wang-etal-2024-arithmetic, xie-etal-2025-bone}. Gradient-level scalarization builds update directions in parameter space, as in GAPO \citep{li-etal-2025-gradient}, which requires one backward pass per objective. Like dynamic reward weighting, \method{} adapts reward-level weights, but it reacts only to a per-objective covariance signal and needs no per-objective gradients.

\subsection{Challenges in Reinforcement Fine-Tuning}
LLM RFT faces several optimization challenges beyond multi-objective settings. Lagrangian dynamics can become unstable when convexity assumptions fail \citep{FEIJER20101974}. LLM-specific challenges include vanishing gradients \citep{razin2024vanishing}, sensitivity to importance weighting and normalization \citep{zheng2025groupsequencepolicyoptimization}, and exploration difficulties \citep{jiang2025rethinkingentropyregularizationlarge}. Recent RL scaling laws further suggest that optimization performance varies with model size \citep{khatri2025artscalingreinforcementlearning}. Closest to our setting is gradient starvation \citep{pezeshki2021gradient}, where a dominant feature captures the gradient at the expense of others. However, its solution, spectral decoupling, collapsed the policy at every strength we tried, even when training on accuracy alone (Appendix~\Cref{tab:sd-probe}).

\section{Preliminaries}
\paragraph{Notation.}
Following the notation of \citet{razin2024vanishing}, let $\mathcal{D}$ be the dataset and $\mathcal{X}$ be a finite token vocabulary. We then define $\mathcal{X}^{L_\mathrm{in}}$ as the space of input prompts of length $L_\mathrm{in}$, and $\mathcal{X}^{L_\mathrm{out}}$ as the space of output sequences of length $L_{\mathrm{out}}$. We study $M$ objectives, and for a given input prompt $\vect{x} = (x_1,x_2,\cdots,x_{L_\mathrm{in}}) \in \mathcal{X}^{L_\mathrm{in}}$ and generated completion
$\vect{y} = (y_1,\dots,y_{L_{\mathrm{out}}}) \in \mathcal{X}^{L_\mathrm{out}}$, the reward function is $r: \mathcal{X}^{L_\mathrm{in}} \times \mathcal{X}^{L_\mathrm{out}} \to \R^M$.

\paragraph{RFT as a contextual bandit.} We model RFT as a horizon-one (bandit) environment, in which each input is a state and each completion an action. An autoregressive language model with parameters $\theta \in \R^n$ induces $p_\theta(\vect{y} \mid \vect{x}) = \prod_{l=1}^{L_{\mathrm{out}}} p_\theta\big(y_l \mid \vect{x}, \vect{y}_{\leq l-1}\big) = \prod_{l=1}^{L_{\mathrm{out}}}\mathrm{softmax}\big(f(\vect{x},\vect{y}_{\leq l-1};\theta)\big)_{y_l}$, where $\vect{y}_{\leq l-1} \coloneq (y_1,\cdots,y_{l-1})$ is the partial completion and $f(\vect{x},\vect{y}_{\leq l-1};\theta) \in \R^{|\mathcal{X}|}$ is the next-token logit map. For each objective $m\in\{1,\ldots,M\}$, the expected objective reward is $r_m(p_\theta) \coloneq \E_{\vect{x}\sim\mathcal{D}}\, \E_{\vect{y}\sim p_\theta(\cdot\mid \vect{x})}\big[r_m(\vect{x},\vect{y})\big]$.

\paragraph{Scalarization.} A scalarization map $\Psi:\R^M\to\R$ converts the vector reward to a per-sample score $s(\vect{x},\vect{y}) \coloneq \Psi(r(\vect{x},\vect{y}))$. The induced value function for input $\vect{x}$ is
\begin{equation}
V(\vect{x};\theta) \coloneq \E_{\vect{y}\sim p_\theta(\cdot\mid \vect{x})}\big[\,s(\vect{x},\vect{y})\,\big],
\label{eq:value function}
\end{equation}
and the overall RFT objective is to maximize $V(\theta) \coloneq \E_{\vect{x}\sim \mathcal{D}}\big[V(\vect{x};\theta)\big]$.

\section{Local Covariance Laws for Multi-Objective Policy Improvement}
\label{sec:local-law}
In this section, we establish sufficient conditions under which optimizing a scalarized score $s(\vect{x},\vect{y})$ guarantees first-order improvement in objectives $r_m(\vect{x},\vect{y})$. We analyze a KL-regularized improvement step in distribution space (\S\ref{subsec:reinforce-kl-cov}), instantiate it in a closed-form two-mode example where scalarization provably hurts an objective (\S\ref{subsec:toy-example}), extend it to clipped surrogate objectives (\S\ref{subsec:grpo-clipping}), and evaluate it empirically (\S\ref{sec:experiments}). Proofs are deferred to Appendix~\ref{appendix:proof}.

\subsection{KL-Regularized Policy Improvement}
\label{subsec:reinforce-kl-cov}

For fixed $\vect{x}$, write $p_{\theta;\vect{x}}(\cdot):=p_\theta(\cdot\mid\vect{x})$ and define the KL-regularized improvement step in distribution space:
\begin{equation}
\label{eq:kl-policy-improvement}
p_{\theta;\vect{x}}^{+}\coloneq
\arg\max_{q \in \Delta(\mathcal{X}^{L_{\mathrm{out}}})}
\Big\{
\E_{\vect{y}\sim q}\big[s(\vect{x},\vect{y})\big]
-\frac{1}{\eta}\,\mathrm{KL}\big(q\|p_{\theta;\vect{x}}\big)
\Big\},
\end{equation}
where $\eta>0$ is the step size and $\Delta(\mathcal{X}^{L_\mathrm{out}})$ is the simplex over completions. Its optimizer is an exponential tilt of the current policy, $p_{\theta;\vect{x}}^{+}(\vect{y})\propto p_{\theta;\vect{x}}(\vect{y})\exp\big(\eta\,s(\vect{x},\vect{y})\big)$ (\Cref{lemma:exp-tilt}, Appendix~\ref{appendix:proof}).

\begin{theorem}[First-order local covariance law]
\label{theorem:local-cov-law}
Assume $r_m(\vect{x},\vect{y})$ is bounded and there exists $\eta_0>0$ such that, for all $\vect{x}\in\mathcal{X}^{L_{\mathrm{in}}}$ and $ |\eta| \le \eta_0$, $\E_{\vect{y}\sim p_{\theta;\vect{x}}}\big[\exp(\eta\,s(\vect{x},\vect{y}))\big] < \infty$. Then for each objective $m\in\{1,\dots,M\}$,
\begin{align}
\label{eq:local-cov-law}
r_m(p_\theta^{+})-r_m(p_\theta)= \eta\E_{\vect{x}\sim \mathcal{D}}
\Big[
\mathrm{Cov}_{\vect{y}\sim p_{\theta}(\cdot\mid \vect{x})}
\big(r_m(\vect{x},\vect{y}),\,s(\vect{x},\vect{y})\big)
\Big]+O(\eta^2).
\end{align}
Consequently, if $\E_{\vect{x}\sim \mathcal{D}}\Big[\mathrm{Cov}_{\vect{y}\sim p_{\theta}(\cdot\mid\vect{x})}\big(r_m(\vect{x},\vect{y}),s(\vect{x},\vect{y})\big)\Big] > 0$, then $r_m(p_\theta^{+})>r_m(p_\theta)$ for sufficiently small $\eta>0$.
\end{theorem}

\Cref{eq:local-cov-law} shows that optimizing the scalar score $s$ improves objective $m$ at first order when completions with higher $s(\vect{x},\vect{y})$ also tend to have higher $r_m(\vect{x},\vect{y})$, leading to positive covariance averaged across prompts. Conversely, negative covariance brings a local tradeoff where increasing $s$ necessarily decreases $r_m$ at first order.

Because covariance is computed on-policy, its sign can flip over training as $p_\theta$ moves. The scalarized update may improve objective $m$ early in training but degrade it later, even when objectives are not inherently conflicting on the global Pareto front.
For linear scalarization $s_\lambda=\sum_j \lambda_j r_j$, the condition becomes $\mathrm{Cov}(r_m,s_\lambda)=\sum_j \lambda_j \mathrm{Cov}(r_m,r_j)$, making cross-objective interference more concrete. Emphasizing an easy objective can flip the sign for a harder one when their rewards are negatively correlated on-policy.

\subsection{A Two-Mode Toy Example: When Scalarization Hurts an Objective}
\label{subsec:toy-example}
We analyze a minimal setting where the completion distribution places most of its mass on two modes (e.g., two distinct styles or solutions that the model frequently samples). Formally, for a fixed prompt $\vect{x}$, we idealize the completion space as two canonical outputs $\mathcal{Y}(\vect{x})=\{\vect{y}_{\mathrm{good}},\ \vect{y}_{\mathrm{bad}}\}$, where ``good'' and ``bad'' are defined with respect to a particular objective $m$, not the scalar score. Let $p_t \coloneq p_{\theta_t;\vect{x}}(\vect{y}_{\mathrm{bad}})$ so that $1-p_t = p_{\theta_t;\vect{x}}(\vect{y}_{\mathrm{good}})$ and define 
\begin{align*}
&s_{\mathrm{good}}\coloneq s(\vect{x},\vect{y}_{\mathrm{good}}),\;
s_{\mathrm{bad}}\coloneq s(\vect{x},\vect{y}_{\mathrm{bad}}),\;
r_{\mathrm{good}}\coloneq r_m(\vect{x},\vect{y}_{\mathrm{good}}),\;
r_{\mathrm{bad}}\coloneq r_m(\vect{x},\vect{y}_{\mathrm{bad}}).
\end{align*}

\paragraph{Optimizing $s$ concentrates probability on what $s$ favors.}
By \Cref{lemma:exp-tilt}, after one KL-regularized improvement step, the probability of $\vect{y}_{\mathrm{bad}}$ becomes $p_{t+1}=(p_t\,e^{\eta s_{\mathrm{bad}}}) / 
(p_t\,e^{\eta s_{\mathrm{bad}}}+(1-p_t)\,e^{\eta s_{\mathrm{good}}})$,
and the log-odds update is
\begin{equation}
\label{eq:log-odds}
\log\frac{p_{t+1}}{1-p_{t+1}}=\log\frac{p_t}{1-p_t}
+\eta\big(s_{\mathrm{bad}}-s_{\mathrm{good}}\big).
\end{equation}
If the proxy score ranks the bad mode higher, $
s_{\mathrm{bad}}>s_{\mathrm{good}}$, then the log-odds increase linearly in $t$, and $p_t\to 1$ (i.e., the update drives mass toward $\vect{y}_{\mathrm{bad}}$).

\paragraph{A biased objective can decrease monotonically over steps.}
Assume the objective $m$ prefers the $\vect{y}_{\mathrm{good}}$ mode, i.e., $r_{\mathrm{good}} > r_{\mathrm{bad}}$. Then the expected objective for input $\vect{x}$ is
\begin{align*}
\E_{\vect{y}\sim p_{\theta_t}(\cdot\mid\vect{x})}\big[r_m(\vect{x},\vect{y})\big]=(1-p_t)\,r_{\mathrm{good}} + p_t\,r_{\mathrm{bad}} =r_{\mathrm{good}} - p_t\,(r_{\mathrm{good}}-r_{\mathrm{bad}}),
\end{align*}
which is strictly decreasing in $p_t$. Combining this with \Cref{eq:log-odds} yields a simple case of cross-objective interference: if the scalarized score favors the mode that is worse for objective $m$ ($s_{\mathrm{bad}}>s_{\mathrm{good}}$) while objective $m$ prefers the other mode ($r_{\mathrm{good}}>r_{\mathrm{bad}}$), then training increases $p_t$ at each step and $\E[r_m(\vect{x},\vect{y})]$ decreases monotonically toward $r_{\mathrm{bad}}$.

\paragraph{The covariance law flags the failure.}
In this two-mode example, the conditional covariance has a closed form:
\begin{align}
\label{eq:toy-cov}
\mathrm{Cov}_{\vect{y}\sim p_{\theta_t}(\cdot\mid\vect{x})}
\big(r_m(\vect{x},\vect{y}),s(\vect{x},\vect{y})\big) =
p_t(1-p_t)\,(r_{\mathrm{good}}-r_{\mathrm{bad}})\,(s_{\mathrm{good}}-s_{\mathrm{bad}}).
\end{align}
Under the interference configuration $s_{\mathrm{bad}}>s_{\mathrm{good}}$ and $r_{\mathrm{good}}>r_{\mathrm{bad}}$, \Cref{eq:toy-cov} is negative when $p_t\in(0,1)$. Therefore, our local covariance law predicts that, for sufficiently small $\eta$, the KL-regularized improvement step that increases $s$ must decrease the true objective $r_m$ at first order.

\subsection{Extension to Clipped Surrogate Objectives}
\label{subsec:grpo-clipping}

The two-mode example is illustrative, but modern RFT methods such as GRPO and PPO \citep{schulman2017proximalpolicyoptimizationalgorithms} optimize clipped surrogate objectives rather than the idealized value function. In this section, we show that the derived covariance law continues to hold under mild conditions. The key observation is that clipping does not introduce a fundamentally different improvement mechanism, and it only removes a subset of weighted logit-gradient terms from the update. Therefore, when the removed gradient mass is sufficiently small, the first-order improvement guarantee is preserved. For simplicity, we defer the full formal analysis to Appendix~\ref{appendix:analysis} and present the simplified corollary below.

\begin{corollary}[Clipping robustness, simplified]
\label{cor:clipping-robustness-simplified}
Suppose the unclipped surrogate update has a positive first-order improvement margin for objective $m$ (its direction locally increases $r_m$), and let the clipping distortion be the Fisher-weighted gradient mass removed by clipping. If the distortion is smaller than the margin, the clipped update still yields $r_m(\theta^{+})\ge r_m(\theta)$ for small enough $\eta>0$.
\end{corollary}

We also find that in the categorical bandit case, this condition reduces to the same covariance form as \Cref{theorem:local-cov-law} (see Corollary~\ref{cor:categorical-reduction}). Importantly, this clipping distortion can be carefully controlled in practice through techniques such as learning rate scheduling or reward normalization, keeping clipping distortion small enough for the covariance law to hold. We confirm this empirically in Appendix~\Cref{fig:grpo result}. Under GRPO, interference follows the same pattern across models as under unclipped REINFORCE (\Cref{fig:main intro figure}), so the covariance law is robust to clipping.

\subsection{Experiments: Cross-Objective Interference and Covariance Law}
\label{sec:experiments}

We first describe the experimental setup shared across the paper. We then show that cross-objective interference is pervasive yet model-dependent, that it arises from co-optimization, and that the covariance sign flags declining objectives.

\subsubsection{Setup}

\paragraph{Domains, objectives, and rewards.} We use six datasets: MATH-500~\citep{mathlighteval, lightman2024lets} and GSM8K~\citep{cobbe2021trainingverifierssolvemath} for math, IFEval~\citep{zhou2023instructionfollowingevaluationlargelanguage} for instruction following, MBPP~\citep{austin2021programsynthesislargelanguage} for code, PKU-SafeRLHF~\citep{ji2025pkusaferlhfmultilevelsafetyalignment} for safety, and XSum~\citep{narayan-etal-2018-dont} for summarization. Each pairs a task objective with one or two auxiliary objectives, and every reward lies in $[0,1]$ (Appendix~\ref{appendix:rewards}).

\paragraph{Training and evaluation.} Each dataset is split 0.6/0.2/0.2 into train/validation/test. We evaluate on the test set every 5 training steps and report only test results. Results reported with a standard deviation are means over three seeds. The training curves, \Cref{fig:interference-bars}, and Appendix~\Cref{tab:scope,tab:sd-probe} show single runs. Full hyperparameter settings are provided in Appendix~\Cref{table:hyperparameters}. We define an objective suffers interference if it falls from its peak by more than a noise floor and 15\% of its range while another objective improves (\Cref{def:interference}).

\paragraph{Models and baselines.} We train four models at the 1.5B-1.7B scale (Qwen2.5-1.5B-Base, Qwen2.5-1.5B-IFT~\citep{qwen2025qwen25technicalreport}, Qwen3-1.7B-Base~\citep{yang2025qwen3technicalreport}, and SmolLM2-1.7B-Base~\citep{allal2025smollm}) and five at the 7B-8B scale (Qwen2.5-7B-Base, DeepSeek-7B-Chat~\citep{deepseekai2024deepseekllmscalingopensource}, Llama3.1-8B-IFT~\citep{grattafiori2024llama3herdmodels}, Ministral-8B-IFT~\citep{mistral2024ministral8b}, and Qwen3-8B-Base).
We study nine scalarization algorithms, five at the reward level (Linear, Dynamic, Lagrangian, Tchebycheff, PAMA) and four at the gradient level (MGDA, GradNorm, Nash-MTL, FAMO), which also serve as baselines for \method{} (\S\ref{sec:ctwa}).
Appendix~\ref{appendix:algorithms} gives pseudocode for baselines adapted to LLM alignment. All main experiments use unclipped REINFORCE++ (REINFORCE hereafter, \citealp{hu2025reinforcestabilizingcriticfreepolicy}), and we rerun MATH-500 under clipped GRPO.

\begin{table}[t]
  \centering
  \small
  \setlength{\tabcolsep}{2pt}
  \renewcommand{\arraystretch}{1.15}
  \caption{\textbf{Per-objective covariance and test-performance change} on MATH-500 (Qwen2.5-1.5B-Base) and IFEval (Qwen2.5-1.5B-IFT). \emph{Mean covariance}: $\E_{t,\vect{x}}\big[\mathrm{Cov}_{\vect{y}}(r_m, s_{\vect{\lambda}_t})\big]$ over training steps $t$ and prompts $\vect{x}$, with $\vect{\lambda}_t$ the method's weights at step $t$. \emph{Negative fraction}: fraction of training prompts with a negative covariance. \emph{Test performance} $\Delta_m$: mean of the last five test evaluations minus initial performance. $^\dagger$Conciseness starts at its ceiling of 1.0, so $\Delta_m\le0$ by design (Appendix~\ref{appendix:rewards}). Under all three methods, helpfulness has the smallest mean covariance and the largest negative fraction, and it is the only objective besides conciseness that does not improve.}
  \label{tab:cov-per-method}
  \resizebox{\textwidth}{!}{%
  \begin{tabular}{@{}ll rrr rrr rrr@{}}
  \toprule
  & & \multicolumn{3}{c}{Mean covariance ($\times 10^{-2}$)} & \multicolumn{3}{c}{Negative fraction} & \multicolumn{3}{c}{Test performance $\Delta_m$} \\
  \cmidrule(lr){3-5}\cmidrule(lr){6-8}\cmidrule(l){9-11}
  Dataset & Objective & \multicolumn{1}{c}{Linear} & \multicolumn{1}{c}{Dynamic} & \multicolumn{1}{c}{\method{}} & \multicolumn{1}{c}{Linear} & \multicolumn{1}{c}{Dynamic} & \multicolumn{1}{c}{\method{}} & \multicolumn{1}{c}{Linear} & \multicolumn{1}{c}{Dynamic} & \multicolumn{1}{c}{\method{}} \\
  \midrule
  \multirow{3}{*}{MATH-500}
   & Accuracy    & $2.40$\sd{0.09} & $2.40$\sd{0.08} & $2.70$\sd{0.16} & $0.000$\sd{0.000} & $0.000$\sd{0.000} & $0.000$\sd{0.000} & $+0.34$\sd{0.05} & $+0.34$\sd{0.02} & $+0.33$\sd{0.02} \\
   & Clarity     & $0.91$\sd{0.14} & $0.83$\sd{0.05} & $0.93$\sd{0.05} & $0.001$\sd{0.000} & $0.001$\sd{0.000} & $0.001$\sd{0.000} & $+0.96$\sd{0.01} & $+0.96$\sd{0.01} & $+0.97$\sd{0.00} \\
   & Conciseness$^\dagger$ & $1.40$\sd{0.20} & $1.34$\sd{0.05} & $1.49$\sd{0.09} & $0.002$\sd{0.000} & $0.003$\sd{0.000} & $0.004$\sd{0.001} & $-0.03$\sd{0.01} & $-0.05$\sd{0.01} & $-0.06$\sd{0.02} \\
  \midrule
  \multirow{3}{*}{IFEval}
   & Constraint following & $1.72$\sd{0.02} & $1.75$\sd{0.03} & $1.81$\sd{0.01} & $0.005$\sd{0.001} & $0.004$\sd{0.001} & $0.004$\sd{0.001} & $+0.19$\sd{0.04} & $+0.20$\sd{0.03} & $+0.21$\sd{0.02} \\
   & Helpfulness & $0.03$\sd{0.01} & $0.04$\sd{0.03} & $0.06$\sd{0.02} & $0.416$\sd{0.003} & $0.420$\sd{0.017} & $0.419$\sd{0.007} & $-0.04$\sd{0.03} & $-0.04$\sd{0.02} & $-0.03$\sd{0.03} \\
   & Conciseness$^\dagger$ & $1.75$\sd{0.03} & $1.81$\sd{0.01} & $1.84$\sd{0.03} & $0.006$\sd{0.000} & $0.006$\sd{0.001} & $0.007$\sd{0.001} & $-0.10$\sd{0.03} & $-0.10$\sd{0.03} & $-0.11$\sd{0.02} \\
  \bottomrule
  \end{tabular}%
  }
  \vspace{-6pt}
\end{table}

\subsubsection{Results and Analysis}
\label{subsec:phenomenon}

\paragraph{Interference is pervasive yet model-dependent.} Under \Cref{def:interference}, 19 of the 30 REINFORCE runs on MATH-500 interfere, 18 of them on the two Qwen2.5 models, and accuracy is always the degraded objective.
\textbf{Domains.} Under Linear on Qwen2.5-1.5B-IFT, IFEval helpfulness ends 0.04 below its starting value while constraint following gains 0.19 (\Cref{tab:cov-per-method}). MBPP pass rate falls ($0.565\!\to\!0.424$) as programs shorten from 73 to 19 tokens, and XSum coverage collapses ($0.241\!\to\!0.054$) as summaries shrink from 44 to about 4 tokens and faithfulness rises.
% In contrast, GSM8K (on Qwen2.5-1.5B-Base) and PKU-SafeRLHF show no interference (Appendix~\Cref{tab:scope}).
\textbf{Scale.} At 7B and 8B, accuracy again falls from its peak on Qwen2.5-7B-Base ($0.67\!\to\!0.51$ under Dynamic) and Ministral-8B-IFT while clarity saturates, whereas Qwen3-8B-Base co-improves like its 1.7B counterpart (Appendix~\Cref{fig:large models}).
\textbf{Clipping.} Under GRPO on MATH-500, \Cref{def:interference} flags 21 of the 30 runs, all but one of them on the Qwen2.5 models, so clipping does not change which models suffer interference (Appendix~\Cref{fig:grpo result}).

\begin{wraptable}{r}{0.33\textwidth}
\vspace{-12pt}
\small
\centering
% \small\setlength{\tabcolsep}{6pt}\renewcommand{\arraystretch}{1.1}
\caption{\textbf{Single- vs.\ multi-objective training on MATH-500} with Qwen2.5-1.5B-Base. \emph{Final accuracy}: mean of the last five test evaluations.}
\label{tab:single-obj}
\vspace{-6pt}
\begin{tabular}{@{}lc@{}}
\toprule
Objectives & Final accuracy \\
\midrule
Accuracy only & $0.535$\sd{0.017} \\
Clarity only & $0.134$\sd{0.011} \\
Conciseness only & $0.000$\sd{0.000} \\
\midrule
All three (Linear) & $0.359$\sd{0.050} \\
\bottomrule
\end{tabular}
\vspace{-20pt}
\end{wraptable}
\paragraph{Interference arises from co-optimization.} The accuracy drop could come from co-optimizing the objectives or simply from the instability of RFT, which would arise even with a single objective. To disentangle the two, we train Qwen2.5-1.5B-Base on each objective alone (\Cref{tab:single-obj}). Trained alone, accuracy ends at 0.535, well above the 0.359 under Linear. The two settings are close early in training (0.51 vs.\ 0.47 over steps 20-40), but accuracy-only training then drops just 0.059 from its peak, against 0.141 under Linear (Appendix~\Cref{tab:single-obj-full}). Most of the drop therefore comes from co-optimization, not from the instability of RFT.

\paragraph{Covariance flags some, but not all, declining objectives.} \Cref{tab:cov-per-method} tests whether frequent negative covariance predicts a decline at test time. On IFEval, helpfulness has by far the largest negative fraction (0.42) and is the only objective besides conciseness to decline. On MATH-500, negative fractions stay below 0.01, and accuracy and clarity end well above their initial values ($+0.34$ and $+0.96$ under Linear). The pattern is nearly identical under all three methods, \method{} included, so it does not depend on which of these methods sets the weights.
That said, the covariance sign does not explain every drop from the peak. Despite its net gain, MATH-500 accuracy falls from its peak (0.50 to 0.36 under Linear, Appendix~\Cref{tab:single-obj-full}), and XSum coverage collapses (Appendix~\Cref{tab:scope}), although their negative fractions are near zero (0.000 and 0.019).
These mismatches happen because the law is local and explains only a single update step, which motivates the global convergence analysis in \S\ref{sec:pl}.

\section{COVER: COVariance-floor Enforced Reweighting}
\label{sec:ctwa}

Motivated by the theoretical analysis of \S\ref{sec:local-law}, we propose \textit{COVariance-floor Enforced Reweighting (\method)}, a plug-and-play, one-sided covariance controller. It raises an objective's scalarization weight only when the covariance between its reward and the clipped advantage weight (defined below) falls below a target, and otherwise leaves the weights unchanged (Appendix~\Cref{alg:ctwa}).

Following \citet{lu2026learningoptimizemultiobjectivealignment}, we use a weighted sum as the scalarization function in our experiments, $s_{\lambda}(\vect{x},\vect{y}) \coloneq \sum_{m=1}^M \lambda_m\, r_m(\vect{x},\vect{y})$. Let $w^{\mathrm{clip}}_{k,l}(\theta)$ denote the tokenwise clipped advantage weight (defined in Corollary \ref{cor:clipping-robustness}). We compute it under the updated policy $\theta$ on the $K$ sampled completions. It equals the advantage of completion $k$ multiplied by the importance ratio of its token $l$. We aggregate $w^{\mathrm{clip}}_{k,l}(\theta)$ over tokens to obtain the completion-level clipped advantage weight $w\left(\vect{x},\vect{y}^{(k)};\theta\right)
\coloneq \frac{1}{L_{\mathrm{out}}}\sum_{l=1}^{L_{\mathrm{out}}} w^{\mathrm{clip}}_{k,l}(\theta)$. For each prompt $\vect{x}$ and objective $m$, \method{} first computes the within-prompt empirical covariance $\widehat{\mathrm{Cov}}_m(\vect{x})
\coloneq \mathrm{Cov}_{k=1\cdots K}\!\Big( r_m(\vect{x},\vect{y}^{(k)}), w(\vect{x},\vect{y}^{(k)};\theta)\Big)$, and then averages it over the prompts in the batch to obtain $c_m \coloneq \E_{\vect{x}\ \text{in batch}}\big[\widehat{\mathrm{Cov}}_m(\vect{x})\big]$.

In practice, \method{} treats $c_m$ as a diagnostic of whether the current scalarized update benefits objective $m$ and uses it to adjust $\lambda_m$ for the next policy update. Because the covariance of a single batch is noisy, \method{} maintains an exponential moving average (EMA) of it, $\bar{c}_m \leftarrow (1-\tau)\bar{c}_m + \tau c_m$, and measures how far this average lies below the target $c^\ast_m$ with the non-negative deficit $\delta_m \coloneq \big[c^\ast_m - \bar{c}_m\big]_+$.
To keep $\lambda_m>0$ and make the updates multiplicative, \method{} sets $\lambda_m=\exp(u_m)$ and updates $u_m \leftarrow u_m + \eta_\lambda\,\delta_m$ with learning rate $\eta_\lambda$.
Unlike Dynamic or MGDA, \method{} requires no per-objective gradients or projected gradient descent. The covariance statistics are computed alongside the standard RFT update with negligible overhead.

\method{} acts as a one-sided covariance controller. Because $\delta_m\ge0$, the log-weight $u_m$ never decreases, so \method{} raises $\lambda_m$ only while $\bar{c}_m$ is below $c^\ast_m$. Two properties follow directly. \textbf{(1) Linear by default.} While $\bar{c}_m\ge c^\ast_m$ for all $m$, the deficits vanish and the weights stop changing, so \method{} acts as linear scalarization with its current weights. \textbf{(2) Early activation.} A positive target lets \method{} raise a weight before the covariance turns negative, which matters in practice given that a small positive covariance means only slow improvement that noise can easily reverse.

\begin{figure}[t]
  \centering
  \includegraphics[width=\textwidth]{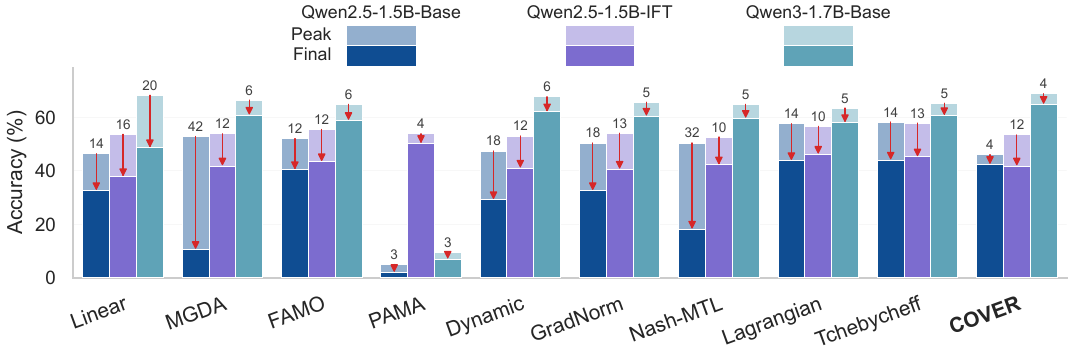}
  \vspace{-15pt}
  \caption{\textbf{Test accuracy for ten methods across three models on MATH-500.} Final accuracy averages the last five evaluations up to step 560, and the \textcolor{figred}{$\downarrow$} and the number above it mark the drop, in percentage points, from the peak. \textbf{Across models, \method{} matches the strongest baselines in final accuracy with one of the smallest drops from the peak.}}
  \label{fig:interference-bars}
  \vspace{-12pt}
\end{figure}

\subsection{Experiments: Evaluating COVER}
\label{sec:ctwa-exp}

We first test whether \method{} reduces the accuracy drop on MATH-500 and how the drop depends on the covariance target. We then check that \method{} matches Linear when objectives already co-improve, and finally turn to other domains and to larger models.

\paragraph{\method{} reduces the accuracy drop on all three models.} \Cref{fig:interference-bars} compares the ten methods under REINFORCE. On both Qwen2.5 models, every baseline except PAMA loses at least 0.10 from its peak accuracy, up to 0.42 for MGDA on Qwen2.5-1.5B-Base. \method{} loses 0.04 on Qwen2.5-1.5B-Base and 0.12 on Qwen2.5-1.5B-IFT, against 0.14 and 0.16 for Linear.
On Qwen2.5-1.5B-Base, only Tchebycheff and Lagrangian end above \method{} (0.44 vs.\ 0.42), but their responses are three to four times as long (\Cref{fig:intro-qwen2.5-base}).
Qwen3-1.7B-Base gives the cleanest comparison, since Linear is the only method that interferes there (under \Cref{def:interference}) and \method{} differs from it only by its controller. Linear loses 0.20 from its peak as its responses shrink to about 210 tokens, whereas \method{} loses 0.04 and reaches the highest final accuracy of 0.65 (\Cref{fig:reinforce-qwen3-base}).

\begin{figure}[t]
  \centering
  \begin{minipage}[t]{0.43\textwidth}
    \centering
    \vspace{-2pt}
    \includegraphics[width=\linewidth]{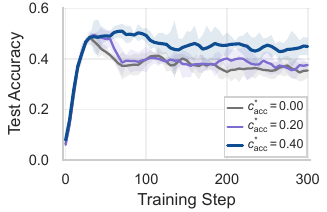}
    \vspace{-20pt}
    \caption{\textbf{Test accuracy for three covariance targets} on MATH-500 with Qwen2.5-1.5B-Base.}
    \label{fig:dose-response}
  \end{minipage}
  \hspace{0.05\textwidth}%
  \begin{minipage}[t]{0.43\textwidth}
    \centering
    \captionof{table}{\textbf{Effect of \method{}'s covariance target on accuracy} on MATH-500 with Qwen2.5-1.5B-Base. \emph{Final accuracy}: mean of the last five evaluations. \emph{Drop}: peak minus final accuracy.}
    \label{tab:cov-sweep}
    \vspace{-5pt}
    \small
    \setlength{\tabcolsep}{4.5pt}\renewcommand{\arraystretch}{1.1}
    \begin{tabular}{@{}lcc@{}}
    \toprule
    Target $c^\ast_{\text{acc}}$ & Final accuracy $\uparrow$ & Drop $\downarrow$ \\
    \midrule
    $0.00$ (Linear) & $0.358$\sd{0.033} & $0.149$\sd{0.032} \\
    $0.15$          & $0.351$\sd{0.032} & $0.165$\sd{0.053} \\
    $0.20$          & $0.389$\sd{0.017} & $0.128$\sd{0.017} \\
    $0.30$          & $0.383$\sd{0.027} & $0.117$\sd{0.041} \\
    $0.40$          & $0.435$\sd{0.042} & $0.101$\sd{0.018} \\
    \bottomrule
    \end{tabular}
  \end{minipage}
  \vspace{-12pt}
\end{figure}

\paragraph{Higher targets reduce the accuracy drop.} We raise $c^\ast_{\text{acc}}$ from the usual range of $\bar{c}_{\text{acc}}$ (0.05-0.16) to well above it (\Cref{tab:cov-sweep,fig:dose-response}).
At 0.15, inside this range, the drop stays at the level of Linear. With targets of 0.20 and above, $\bar{c}_{\text{acc}}$ never reaches the target, so the controller raises the accuracy weight on every step. Accuracy therefore stays near its peak for longer, falling right after step 30 under Linear but only after about step 55 at 0.20 and step 100 at 0.40 (\Cref{fig:dose-response}). At 0.40, the accuracy weight reaches 74\% of the total, and final accuracy rises from 0.358 to 0.435, at some cost to conciseness (0.859 vs.\ 0.959 under Linear). Peak accuracy varies little (0.50-0.54), so most of the accuracy gain comes from cutting the drop from 0.149 to 0.101 (full results in Appendix~\Cref{tab:drag-math}).

\paragraph{No harm when objectives co-improve.} When objectives improve together, \method{} matches Linear even though its controller stays active. On GSM8K (Qwen2.5-1.5B-Base), accuracy and clarity rise under Linear, and the negative fraction of every objective is at most 0.001 (Appendix~\Cref{tab:scope}). The controller stays active because a covariance below the target can signal a plateau as well as a decline. Once accuracy plateaus, its covariance shrinks, and the controller raises the accuracy weight on 77\% of steps, taking it from 33\% to 44\% of the total. Even so, accuracy ends close to Linear (0.830 vs.\ 0.848), clarity is 1.000 under both, and conciseness is higher (0.932 vs.\ 0.884). The same holds on MATH-500 under GRPO with Qwen3-1.7B-Base, where Linear does not interfere and \method{} matches it at 0.62 (Appendix~\Cref{fig:grpo result}). In both settings, no objective ends over 0.02 below Linear.

\paragraph{Other domains and 7B scale.} \method{} works best when an objective declines gradually, since the controller then has time to shift weight toward it. 
On MBPP (Qwen2.5-1.5B-IFT), \method{} raises the pass-rate weight from 50\% to 63\% of the total and cuts the drop from 0.141 to 0.090, at some cost to conciseness (0.720 vs.\ 0.864, Appendix~\Cref{tab:scope}).
On IFEval, where most baselines and fixed non-uniform weights trade one objective for another, \method{} matches Linear and Dynamic to within 0.03 on every objective (Appendix~\Cref{tab:ifeval3,tab:ablation}). On Qwen2.5-7B-Base, only the 0.10 target both reliably reduces the drop and ends above Linear (Appendix~\Cref{tab:7b-sweep}).
However, \method{} has limits. It neither prevents nor reverses the XSum coverage collapse, and targets above 0.10 do not reliably avoid the accuracy collapse that also affects Linear at 7B (Appendix~\Cref{tab:scope,tab:7b-sweep}).

\section{Global Convergence of Multi-Objective Alignment via $\mu$-PL Condition}
\label{sec:pl}

While \S\ref{sec:local-law} characterizes when individual objectives improve locally, it does not explain the model dependence observed in practice. We therefore study the global geometry of the scalarization-induced value $V(\vect{x};\theta)$. Since $V$ is highly non-convex, we ask \textit{when the scalarized RFT objective satisfies the Polyak--\L{}ojasiewicz (PL) condition}, a benign non-convex structure that still yields convergence guarantees \citep{pmlr-v119-mei20b}. The resulting PL constant $\mu$ also suggests a second interference mechanism. When the geometry is unfavorable (small $\mu$), optimization can stall near suboptimal solutions that favor easy objectives.

The analysis rests on three assumptions (formal statements in Appendix~\ref{appendix:proof-pl}): bounded scores, non-saturated token probabilities, and positively aligned token-level gradients. The first two are mild in our setting (rewards lie in $[0,1]$, and KL regularization prevents saturation). The third is stronger and model-dependent. It requires that token-level gradient contributions along a favorable trajectory reinforce rather than cancel each other. Together with the conditioning of the logit Jacobian, this alignment is how model geometry enters $\mu$.

\begin{theorem}[$\mu$-PL condition of multi-objective alignment, simplified]
\label{theorem:pl-informal}
Under \Cref{assumption:bounded-score,assumption:ns,assumption:aligned-gradients}, for every prompt $\vect{x}$ the scalarized value satisfies $\frac{1}{2}\|\nabla_\theta V(\vect{x};\theta)\|^2 \ge \mu\,\big(V(\vect{x};\theta^\ast) - V(\vect{x};\theta)\big)$ with $\mu = \frac{1}{2B}\big(\frac{p_\theta(\vect{y}^\ast\mid\vect{x})}{1-p_\theta(\vect{y}^\ast\mid \vect{x})}\,s(\vect{x},\vect{y}^\ast)\,\gamma - 2B\sigma_\mathrm{max}\big)$, where $\gamma$ collects the alignment, non-saturation, and Jacobian-conditioning constants and $\sigma_\mathrm{max}$ is the largest singular value of the logit Jacobian (formal statement and proof: \Cref{theorem:pl for raw}, Appendix~\ref{appendix:proof-pl}).
\end{theorem}

The constant $\mu$ quantifies how sharply the scalarized landscape rises around its maximizer. It is large when (i) the policy places large mass on the optimal completion, so that $\frac{p_\theta(\vect{y}^\ast\mid\vect{x})}{1-p_\theta(\vect{y}^\ast\mid\vect{x})}$ is large, (ii) the optimal scalarized score $s(\vect{x},\vect{y}^\ast)$ is large, and (iii) the logit map is well-conditioned (large $\gamma$, which summarizes the Jacobian and non-saturation constants). Intuitively, $\mu$ reflects the net effect of the aligned ascent signal from the optimal completion minus the worst-case destructive contribution from all other completions. It therefore turns non-positive when the aligned signal is weak (small $p_\theta(\vect{y}^\ast\mid\vect{x})$ or small $\gamma$) or when the logit Jacobian is highly skewed (large $\sigma_{\max}/\sigma_{\min}$), so the same scalarized update can converge on one model and stall on another.

Importantly, a favorable $\mu$ ensures convergence of $V$ but does not by itself prevent cross-objective interference, because $V$ can rise while an individual objective $r_m$ falls whenever the covariance condition from \S\ref{sec:local-law} fails for that objective. Taken together, \S\ref{sec:local-law} and \S\ref{sec:pl} disentangle two complementary mechanisms and suggest a principled recipe: \textit{ensure convergence of the scalarized objective (make $\mu$ positive and sufficiently large) while maintaining per-objective covariance alignment throughout training}, so that increasing $V$ translates into simultaneous improvement across all objectives rather than interference.

\section{Conclusion}
\label{sec:conclusion}

We conducted, to our knowledge, the first systematic study of scalarization algorithms for multi-objective LLM alignment and formalized a persistent failure mode, cross-objective interference. Through rigorous analysis, we identified the conditions under which each objective can be improved at first order and when the scalarized optimization satisfies the PL condition, and together these results uncover the fundamental challenges in multi-objective alignment. Guided by these insights, we proposed \method{}, a lightweight covariance controller that effectively mitigates cross-objective interference and performs on par with linear scalarization when the objectives already co-improve. Limitations and future work are discussed in Appendix~\ref{appendix:limitations}. 

% \section*{AI Use Statement}
% We did not use generative AI tools for research ideation, theoretical development, or methodology design. We used generative AI tools only to polish the wording of the paper for grammar and readability, and to generate visualization code for figures. We have reviewed all AI-polished text and verified that it preserves our intended meaning and claims. We take responsibility for the final content of this work, including text, claims, or figures produced with the aid of generative AI.

% Do {\bf not} include this section in the anonymized submission, only in the final paper. You can use the \texttt{ack} environment provided in the style file to hide this section in the anonymized submission automatically.
% \end{ack}

\bibliography{ref}
\bibliographystyle{iclr2027_conference}

%%%%%%%%%%%%%%%%%%%%%%%%%%%%%%%%%%%%%%%%%%%%%%%%%%%%%%%%%%%%

\newpage
\appendix

% ---- Appendix table of contents (auto-generated; lists appendix sections only) ----
\phantomsection
\addcontentsline{toc}{section}{Appendix}
\etocdepthtag.toc{appendix}
% compact two-level contents for the appendix only
\begingroup
\etocsettagdepth{mainmatter}{none}
\etocsettagdepth{appendix}{subsection}
\etocsetstyle{section}{}{}
  {\vskip3pt plus1pt\noindent\hangindent1.9em\hangafter1
   {\bfseries\makebox[1.9em][l]{\etocnumber}\etocname}\nobreak
   \leaders\hbox to 6pt{\hfil.\hfil}\hfill\nobreak{\bfseries\etocpage}\par}
  {\vskip2pt}
\etocsetstyle{subsection}{}{}
  {\noindent\hspace*{1.9em}\hangindent4.7em\hangafter1
   \makebox[2.8em][l]{\etocnumber}\etocname\nobreak
   \leaders\hbox to 6pt{\hfil.\hfil}\hfill\nobreak\etocpage\par}
  {}
\etocsettocstyle{\section*{Appendix}\vskip4pt}{}
\parindent0pt \parskip0pt
\tableofcontents
\endgroup
\bigskip

\section{Additional Experimental Results}
\label{appendix:experiments}

\subsection{Objectives}
\label{appendix:rewards}

Every reward except harmlessness is rule-based, and harmlessness is scored by a frozen toxicity classifier. We avoid LLM judges and learned preference models because their scoring errors would add noise to the rewards. All rewards take values in $[0,1]$, and Linear assigns equal weights to the objectives. For each dataset, the first objective listed below is its task objective.

\paragraph{MATH-500.} All three rewards are binary and follow \citet{lu2026learningoptimizemultiobjectivealignment}. \emph{Accuracy} extracts the last \texttt{\textbackslash boxed\{...\}} answer of the response and returns 1 if it matches the reference, and 0 otherwise. \emph{Clarity} returns 1 if the response contains a non-empty \texttt{\textbackslash boxed\{...\}} answer and shows step-by-step structure, meaning at least two enumerated items or bullets, or at least two sequence words such as \emph{first} and \emph{then}, and 0 otherwise. \emph{Conciseness} returns 1 if the response is shorter, in tokens, than the running average length of all training rollouts so far, and 0 otherwise. Before the first training step no running average exists, so every response scores 1 and conciseness starts at its ceiling of 1.0. 

\paragraph{IFEval.} \emph{Constraint following} is the fraction of verifiable instructions the response satisfies, using the IFEval checkers. \emph{Helpfulness} is a saturating word count, $\min(\text{words}/150,\,1)$, that serves as a rule-based proxy for response quality. We adopt length as the signal because \citet{singhal2024a} show that most reward gains in helpfulness-oriented RLHF come from longer responses, and that a reward based on length alone recovers most of the improvement of RLHF over supervised fine-tuning. We cap the reward at 150 words so that it credits a developed answer but does not reward verbosity. \emph{Conciseness} uses the same definition as on MATH-500. Helpfulness and conciseness therefore pull response length in opposite directions.

\paragraph{Other datasets.} 
\textbf{GSM8K} uses the three MATH-500 objectives (accuracy, clarity, conciseness) on grade-school word problems. The MATH algebra problems in \Cref{fig:large models} use the same three objectives.
\textbf{MBPP} pairs \emph{pass rate} (1 if the generated program passes all test assertions in a sandbox, 0 otherwise) with conciseness. 
\textbf{PKU-SafeRLHF} combines \emph{harmlessness}, a \emph{helpfulness} proxy $\min(\text{words}/120,1)$, and conciseness. Harmlessness is $1-P(\text{harmful})$, where $P(\text{harmful})$ is the largest label probability of an off-the-shelf toxicity classifier (\texttt{unitary/toxic-bert}). Because the classifier scores toxic language rather than unsafe compliance, harmlessness is already 0.999 before training and has little room to change (\Cref{tab:scope}).
\textbf{XSum} combines \emph{coverage} (ROUGE-L recall against the reference summary), \emph{faithfulness} (ROUGE-L precision against the same reference, a lexical proxy), and conciseness. Coverage and faithfulness pull against each other, because recall rewards adding content while precision penalizes content that is absent from the reference.

\subsection{Cross-Objective Interference}
\label{appendix:interference-class}

\begin{definition}[Operational definition of cross-objective interference]
\label{def:interference}
Objective $m$ suffers \emph{cross-objective interference} if (i) $m$ degrades from its peak by more than both a noise floor, defined as the median absolute deviation of its one-step reward differences, and 15\% of its observed reward range, and (ii) at least one other objective improves after $m$ peaks.
\end{definition}

Condition~(i) alone indicates reward degradation but not interference, and an objective that merely plateaus fails (i). We set the range threshold at 15\% because it falls in a clear gap between accuracy-only and multi-objective training in \Cref{tab:single-obj-full}. Accuracy-only training loses at most 11\% of its accuracy range from peak, while every Linear, Dynamic, and \method{} run in that table loses 23-45\%.

\begin{table}[ht]
\centering
\small
\setlength{\tabcolsep}{5pt}
\renewcommand{\arraystretch}{1.1}
\caption{\textbf{Single- vs.\ multi-objective training on MATH-500} with Qwen2.5-1.5B-Base using REINFORCE for 315 steps. \emph{Initialization}: the base model before training. \emph{Early}: mean test accuracy over steps 20-40, the window in which accuracy peaks under multi-objective training. \emph{Final}: mean of the last five test evaluations. \emph{Drop}: peak minus final accuracy, where the peak is the highest test accuracy during training. \emph{Length}: mean response length in tokens. Accuracy-only training peaks at 0.59 and loses 0.059, whereas the three multi-objective methods peak at 0.50-0.51 and lose 0.135-0.157 as responses shrink to 112-120 tokens, so co-optimization both lowers the peak and more than doubles the drop. Under conciseness only, accuracy peaks at step 10, before the Early window.}
\label{tab:single-obj-full}
\resizebox{\textwidth}{!}{%
\begin{tabular}{@{}lcccccc@{}}
\toprule
& \multicolumn{3}{c}{Accuracy} & Clarity & Conciseness & Length \\
\cmidrule(lr){2-4} \cmidrule(lr){5-5} \cmidrule(lr){6-6} \cmidrule(lr){7-7}
Objectives & Early & Final $\uparrow$ & Drop $\downarrow$ & Final & Final & Final \\
\midrule
Initialization & \multicolumn{3}{c}{$0.020$} & $0.030$ & $1.000$ & $989$ \\
\midrule
Accuracy only & $0.506$\sd{0.016} & $0.535$\sd{0.017} & $0.059$\sd{0.014} & $0.893$\sd{0.026} & $0.597$\sd{0.019} & $505$\sd{8} \\
Clarity only & $0.416$\sd{0.038} & $0.134$\sd{0.011} & $0.349$\sd{0.015} & $0.948$\sd{0.026} & $0.007$\sd{0.006} & $1020$\sd{5} \\
Conciseness only & $0.003$\sd{0.003} & $0.000$\sd{0.000} & $0.153$\sd{0.060} & $0.000$\sd{0.000} & $1.000$\sd{0.000} & $2$\sd{1} \\
\midrule
All three (Linear) & $0.465$\sd{0.013} & $0.359$\sd{0.050} & $0.141$\sd{0.037} & $0.992$\sd{0.007} & $0.969$\sd{0.008} & $113$\sd{15} \\
All three (Dynamic) & $0.457$\sd{0.010} & $0.361$\sd{0.018} & $0.135$\sd{0.024} & $0.994$\sd{0.009} & $0.949$\sd{0.009} & $112$\sd{6} \\
All three (\method{}) & $0.466$\sd{0.009} & $0.353$\sd{0.019} & $0.157$\sd{0.019} & $0.996$\sd{0.004} & $0.945$\sd{0.024} & $120$\sd{6} \\
\bottomrule
\end{tabular}%
}
\end{table}

\begin{figure}[ht]
  \centering
  \setlength{\tabcolsep}{1pt}
  \renewcommand{\arraystretch}{0}
  \begin{tabular}{ccc}
    \includegraphics[width=0.33\textwidth]{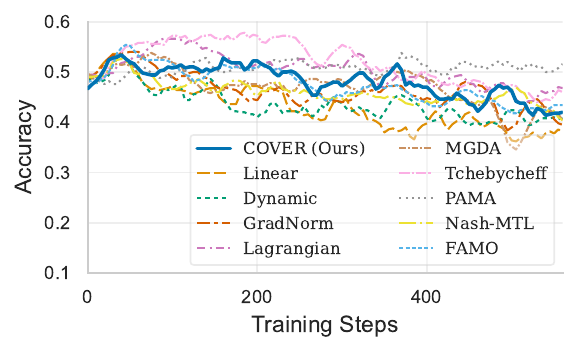} &
    \includegraphics[width=0.33\textwidth]{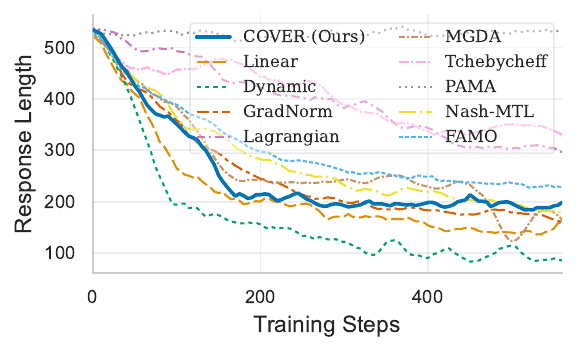} &
    \includegraphics[width=0.33\textwidth]{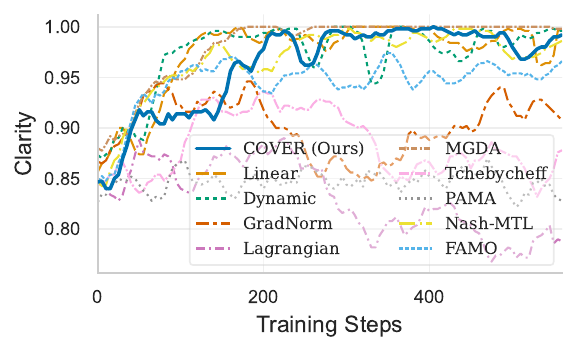}
  \end{tabular}
  \caption{\textbf{MATH-500 test performance over REINFORCE training on Qwen2.5-1.5B-IFT.} Test accuracy, response length in tokens (lower is more concise), and clarity for nine scalarization algorithms and \method{}. Curves are smoothed with a centered moving average over five evaluations. Under every method except PAMA, accuracy peaks at 0.53-0.58 within the first 200 steps and then falls by 0.10-0.16, ending below its starting value (0.38-0.46 vs.\ 0.48). In the same runs, responses shorten from 530 to 90-340 tokens. All ten methods meet \Cref{def:interference}, though PAMA does so only nominally, with a 0.04 fall in accuracy and almost no change in response length or clarity. Clarity also falls from its peak under Lagrangian (by 0.10) and Tchebycheff (by 0.08).}
  \label{fig:reinforce-ift-appendix}
\end{figure}

\begin{figure}[ht]
  \centering
\begin{subfigure}[t]{0.33\textwidth}
  \centering
  \includegraphics[width=\linewidth]{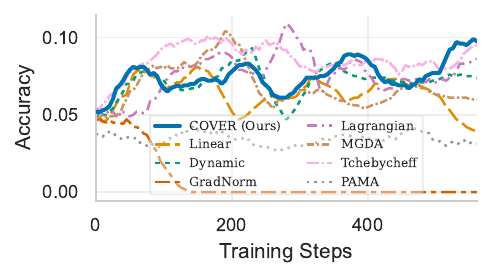}
\end{subfigure}\hfill
\begin{subfigure}[t]{0.33\textwidth}
  \centering
  \includegraphics[width=\linewidth]{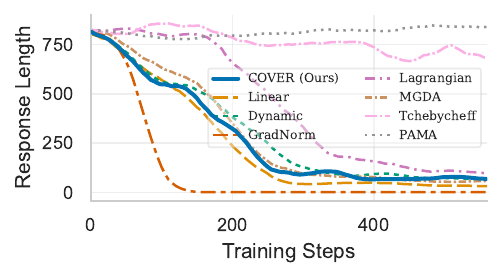}
\end{subfigure}\hfill
\begin{subfigure}[t]{0.33\textwidth}
  \centering
  \includegraphics[width=\linewidth]{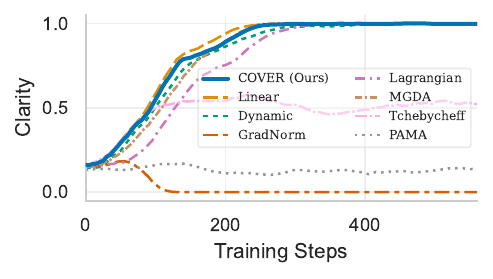}
\end{subfigure}
\caption{\textbf{MATH-500 test performance over REINFORCE training on SmolLM2-1.7B-Base}~\citep{allal2025smollm}.}
\label{fig:smollm-reinforce}
\end{figure}

\begin{figure}[ht]
  \centering
  \setlength{\tabcolsep}{1pt}
  \renewcommand{\arraystretch}{0}

  \begin{subfigure}{\textwidth}
    \centering
    \begin{tabular}{ccc}
      \includegraphics[width=0.33\textwidth]{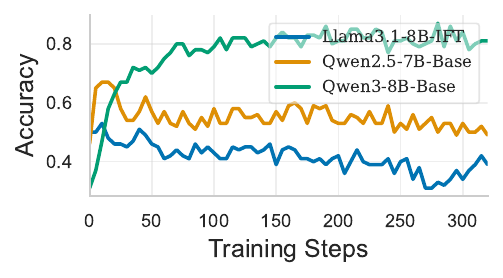} &
      \includegraphics[width=0.33\textwidth]{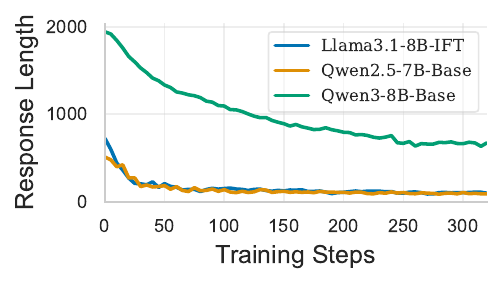} &
      \includegraphics[width=0.33\textwidth]{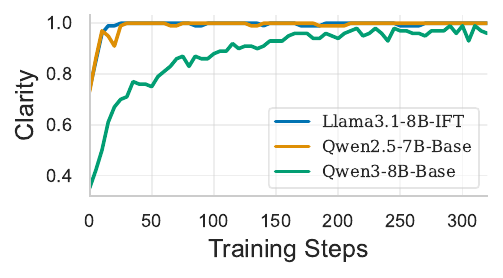}
    \end{tabular}
    \caption{Dynamic reward weighting \citep{lu2026learningoptimizemultiobjectivealignment} on MATH-500 \citep{lightman2024lets}.}
  \end{subfigure}

  \vspace{-2pt}

  \begin{subfigure}{\textwidth}
    \centering
    \begin{tabular}{ccc}
      \includegraphics[width=0.33\textwidth]{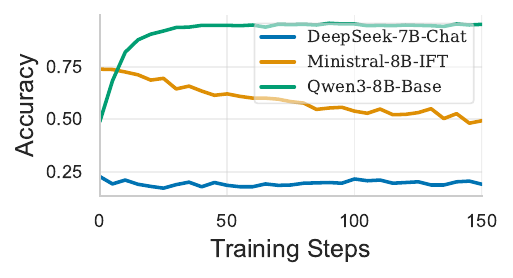} &
      \includegraphics[width=0.33\textwidth]{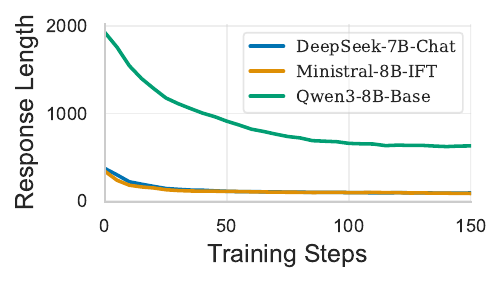} &
      \includegraphics[width=0.33\textwidth]{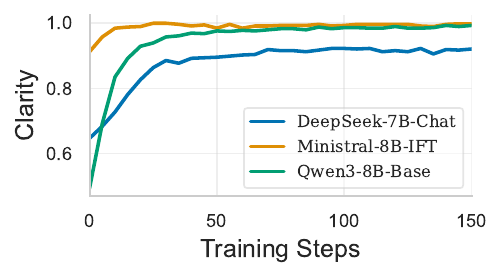}
    \end{tabular}
    \caption{Linear weighting on the MATH algebra problems \citep{mathlighteval}.}
  \end{subfigure}
  \caption{\textbf{Test performance over training at 7B-8B scale.} Columns show accuracy, response length in tokens, and clarity. Row (a) shows three models trained with Dynamic on MATH-500 and row (b) three models trained with Linear on the MATH algebra problems, with Qwen3-8B-Base in both rows. On the other four models, clarity rises to 0.92 or higher and responses shorten to about 100 tokens. Over the same steps, accuracy falls from its early peak on Qwen2.5-7B-Base (0.67 to 0.51) and ends well below its starting value on Llama3.1-8B-IFT (0.50 to 0.38) and Ministral-8B-IFT (0.74 to 0.51). On DeepSeek-7B-Chat, accuracy starts low and dips only slightly (0.23 to 0.19). Qwen3-8B-Base co-improves on both datasets (accuracy 0.31 to 0.80 and 0.49 to 0.95). Its responses also shorten, from about 1900 tokens to about 650, but remain far longer than on the other four models. \textbf{Interference is thus not a small-model artifact.}}
  \label{fig:large models}
\end{figure}

\begin{figure}[ht]
  \centering
\begin{subfigure}[t]{0.33\textwidth}
  \centering
  \includegraphics[width=\linewidth]{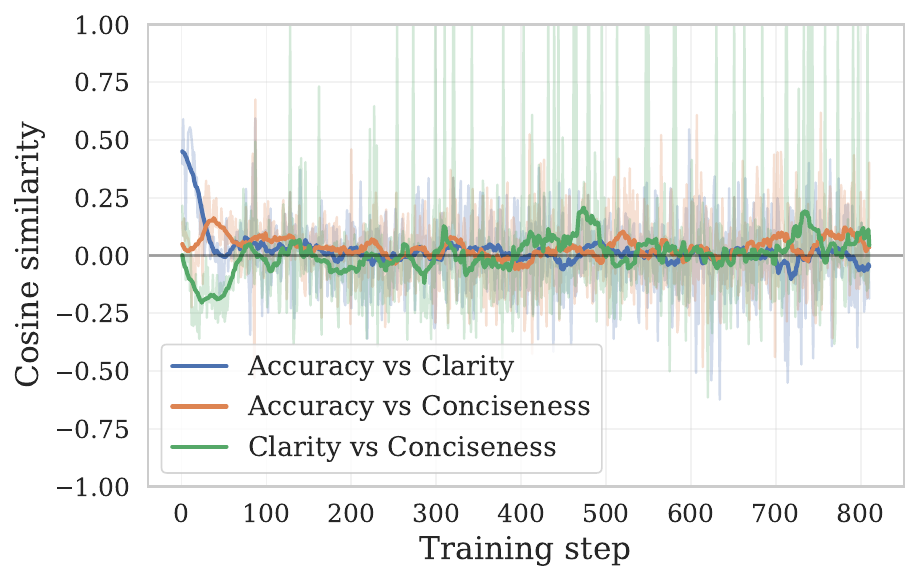}
  \caption{Qwen2.5-1.5B-Base}
\end{subfigure}\hfill
\begin{subfigure}[t]{0.33\textwidth}
  \centering
  \includegraphics[width=\linewidth]{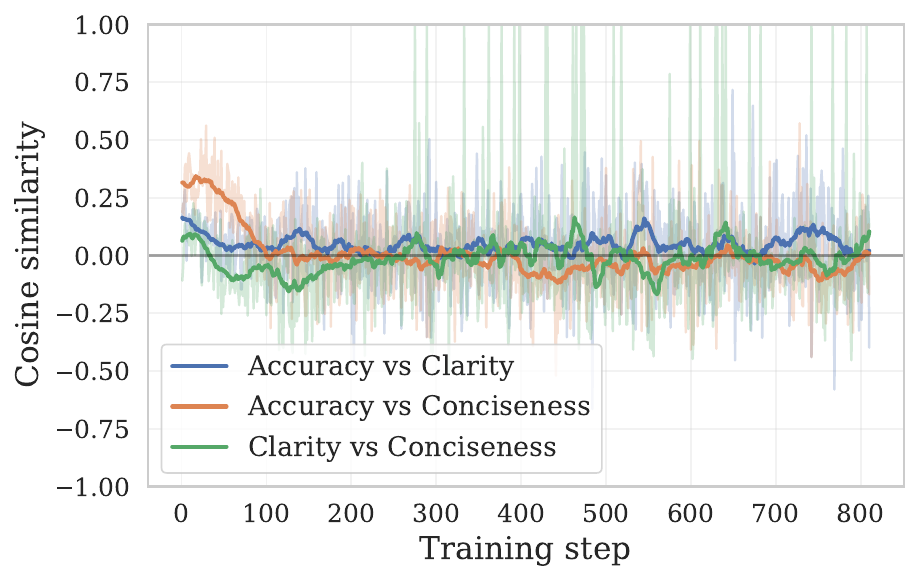}
  \caption{Qwen2.5-1.5B-IFT}
\end{subfigure}\hfill
\begin{subfigure}[t]{0.33\textwidth}
  \centering
  \includegraphics[width=\linewidth]{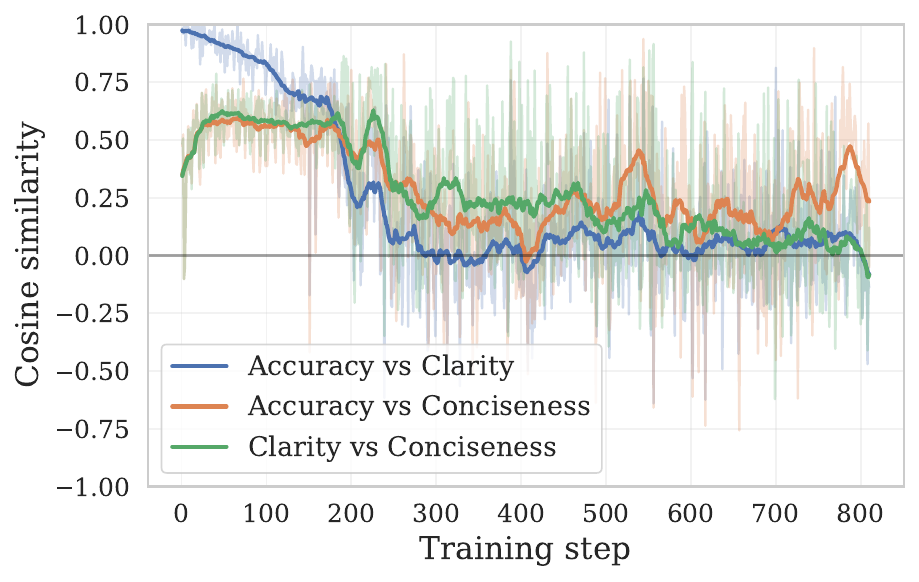}
  \caption{Qwen3-1.7B-Base}
\end{subfigure}
\caption{\textbf{Gradient alignment across objectives during multi-objective alignment.} We measure pairwise cosine similarity between per-objective gradients throughout training. Negative values indicate conflicting updates, which is a standard proxy for identifying conflicting objectives in MTL \citep{NEURIPS2020_3fe78a8a}. Across all three models, cosine similarities remain mostly non-negative and converge toward 0 as training progresses, suggesting that objectives are weakly coupled, neither strongly synergistic nor persistently antagonistic, with no conflicting behavior observed.}
\label{fig:cosine similarity}
\end{figure}

\clearpage
\subsection{\method{} Evaluation}
\label{appendix:drag}

\begin{table}[ht]
\centering
\small
\setlength{\tabcolsep}{5pt}
\renewcommand{\arraystretch}{1.1}
\caption{\textbf{Effect of \method{}'s covariance target on accuracy} on MATH-500 using Qwen2.5-1.5B-Base, extending \Cref{tab:cov-sweep} with additional targets. \emph{Accuracy weight}: share of the total weight at the end of training (33\% at the start). \emph{Final}: mean of the last five test evaluations. \emph{Drop}: peak minus final accuracy. \emph{Length}: response length in tokens. Targets of 0.05-0.15 lie within the usual range of $\bar{c}_{\text{acc}}$ (0.05-0.16), so the accuracy weight rises little and the drop stays at the level of Linear. Targets of 0.20 and above are never reached, so the weight rises on every step and the drop shrinks. Conciseness and response length stay at the level of Linear up to a target of 0.20 and worsen at 0.30 and 0.40.}
\label{tab:drag-math}
\begin{tabular}{@{}lcccccc@{}}
\toprule
& Accuracy weight & \multicolumn{3}{c}{Accuracy} & Conciseness & Length \\
\cmidrule(lr){2-2} \cmidrule(lr){3-5} \cmidrule(lr){6-6} \cmidrule(lr){7-7}
Target $c^\ast_{\text{acc}}$ & Final (\%) & Peak & Final $\uparrow$ & Drop $\downarrow$ & Final $\uparrow$ & Final $\downarrow$ \\
\midrule
$0.00$ (Linear)    & $33$ & $0.507$\sd{0.018} & $0.358$\sd{0.033} & $0.149$\sd{0.032} & $0.959$\sd{0.017} & $116$\sd{11} \\
$0.05$             & $33$ & $0.520$\sd{0.010} & $0.369$\sd{0.029} & $0.151$\sd{0.028} & $0.975$\sd{0.011} & $107$\sd{3} \\
$0.10$             & $34$ & $0.510$\sd{0.044} & $0.352$\sd{0.018} & $0.158$\sd{0.042} & $0.968$\sd{0.006} & $108$\sd{6} \\
$0.15$             & $38$ & $0.517$\sd{0.021} & $0.351$\sd{0.032} & $0.165$\sd{0.053} & $0.979$\sd{0.016} & $106$\sd{7} \\
$0.20$             & $46$ & $0.517$\sd{0.012} & $0.389$\sd{0.017} & $0.128$\sd{0.017} & $0.946$\sd{0.010} & $118$\sd{11} \\
$0.30$             & $61$ & $0.500$\sd{0.020} & $0.383$\sd{0.027} & $0.117$\sd{0.041} & $0.863$\sd{0.011} & $142$\sd{15} \\
$0.40$             & $74$ & $0.537$\sd{0.025} & $0.435$\sd{0.042} & $0.101$\sd{0.018} & $0.859$\sd{0.071} & $189$\sd{51} \\
\bottomrule
\end{tabular}
\end{table}

\begin{table}[ht]
\centering
\small
\setlength{\tabcolsep}{8pt}
\renewcommand{\arraystretch}{1.1}
\caption{\textbf{Multi-objective IFEval results} on Qwen2.5-1.5B-IFT (mean of the last five evaluations). Helpfulness rewards length up to 150 words, while conciseness rewards responses shorter than the running average length of the method's own training rollouts (Appendix~\ref{appendix:rewards}). Lagrangian reaches the highest constraint following but violates its own 0.85 conciseness constraint. MGDA and FAMO saturate helpfulness at the cost of constraint following. Tchebycheff collapses to one-token responses and PAMA to maximum-length responses, so each fails on two objectives. \textbf{Only Linear, Dynamic, and \method{} keep all three objectives above 0.69, and these three differ by at most 0.03 on each objective.}}

\label{tab:ifeval3}
\begin{tabular}{@{}lccc@{}}
\toprule
Method & Helpfulness $\uparrow$ & Constraint following $\uparrow$ & Conciseness $\uparrow$ \\
\midrule
Linear       & $0.715$\sd{0.032} & $0.693$\sd{0.039} & $0.904$\sd{0.027} \\
Dynamic      & $0.718$\sd{0.021} & $0.707$\sd{0.027} & $0.899$\sd{0.025} \\
Lagrangian   & $0.736$\sd{0.028} & $0.755$\sd{0.026} & $0.582$\sd{0.033} \\
Tchebycheff  & $0.000$\sd{0.000} & $0.241$\sd{0.000} & $1.000$\sd{0.000} \\
PAMA         & $0.969$\sd{0.040} & $0.258$\sd{0.064} & $0.068$\sd{0.094} \\
MGDA         & $1.000$\sd{0.000} & $0.541$\sd{0.026} & $0.959$\sd{0.019} \\
GradNorm     & $0.921$\sd{0.123} & $0.675$\sd{0.088} & $0.772$\sd{0.193} \\
Nash-MTL     & $0.628$\sd{0.307} & $0.674$\sd{0.072} & $0.961$\sd{0.024} \\
FAMO         & $0.998$\sd{0.002} & $0.485$\sd{0.016} & $0.987$\sd{0.010} \\
\midrule
\method{}    & $0.729$\sd{0.026} & $0.714$\sd{0.018} & $0.890$\sd{0.024} \\
\bottomrule
\end{tabular}
\end{table}

\begin{table}[ht]
\centering
\small
\setlength{\tabcolsep}{6pt}
\renewcommand{\arraystretch}{1.1}
\caption{\textbf{Static-weight ablation on IFEval,} comparing fixed weight vectors against \method{} on Qwen2.5-1.5B-IFT. The reported numbers are the mean of the last five evaluations. Weight vectors are ordered as helpfulness, constraint following, and conciseness. The vector $[.33,.34,.33]$ is the Linear baseline, and $[.34,.35,.31]$ fixes each weight at the mean of its normalized value over training under \method{}. The vectors $[.45,.30,.25]$ and $[.55,.25,.20]$ move progressively more weight to helpfulness, which declines under Linear (\Cref{tab:cov-per-method}).
Every fixed vector that departs from Linear raises helpfulness, and none raises constraint following above Linear. In contrast, \method{} reaches the highest constraint following in this table, above its own average weights held fixed (0.714 vs.\ 0.676). This suggests that the timing of its reweighting matters, not only its average weights.}
\label{tab:ablation}
\begin{tabular}{@{}lccc@{}}
\toprule
Weights & Helpfulness $\uparrow$ & Constraint following $\uparrow$ & Conciseness $\uparrow$ \\
\midrule
$[.33,.34,.33]$      & $0.715$\sd{0.032} & $0.693$\sd{0.039} & $0.904$\sd{0.027} \\
$[.34,.35,.31]$      & $0.764$\sd{0.001} & $0.676$\sd{0.020} & $0.863$\sd{0.041} \\
$[.45,.30,.25]$      & $0.954$\sd{0.011} & $0.626$\sd{0.003} & $0.923$\sd{0.019} \\
$[.55,.25,.20]$      & $0.993$\sd{0.003} & $0.612$\sd{0.010} & $0.957$\sd{0.013} \\
\midrule
\method{} (Adaptive) & $0.729$\sd{0.026} & $0.714$\sd{0.018} & $0.890$\sd{0.024} \\
\bottomrule
\end{tabular}
\end{table}

\clearpage
\subsection{Generalization and Robustness}
\label{appendix:scope}

\begin{table}[ht]
  \centering
  \small
  \setlength{\tabcolsep}{3.2pt}
  \renewcommand{\arraystretch}{1.15}
  \caption{\textbf{Linear and \method{} on four additional datasets.} GSM8K uses Qwen2.5-1.5B-Base, while MBPP, PKU-SafeRLHF, and XSum use Qwen2.5-1.5B-IFT. \emph{Mean covariance} (in units of $10^{-2}$) and \emph{Negative fraction} are defined as in \Cref{tab:cov-per-method}. \emph{Start} is the first evaluation, \emph{Peak} the best, and \emph{Final} the mean of the last five. Start separates objectives that never rose (MBPP pass rate, XSum coverage) from those that rose and plateaued (GSM8K accuracy, XSum faithfulness). \emph{Engaged} is the fraction of training steps on which \method{} raised the objective's weight, and \emph{Final $\lambda_m$} is its unnormalized weight at the end of training (initially 0.33, or 0.5 on MBPP). $^\dagger$Conciseness is scored against the running average response length (Appendix~\ref{appendix:rewards}), so it starts at 1.0 and can fall even as responses shorten. \textbf{Under Linear, interference appears on MBPP and XSum, where pass rate and coverage decline from their starting values, but not on GSM8K or PKU-SafeRLHF. \method{} reduces the MBPP pass-rate drop from 0.141 to 0.090 but does not prevent the XSum coverage collapse.}}
  \label{tab:scope}
  \resizebox{\textwidth}{!}{%
  \begin{tabular}{@{}llcccccccccc@{}}
  \toprule
  & & Mean covariance & Negative fraction & \multicolumn{6}{c}{Test performance} & \multicolumn{2}{c}{\method} \\
  \cmidrule(lr){3-3}\cmidrule(lr){4-4}\cmidrule(lr){5-10}\cmidrule(l){11-12}
  & & \multirow{2}{*}[-2.7pt]{Linear} & \multirow{2}{*}[-2.7pt]{Linear} & \multicolumn{3}{c}{Linear} & \multicolumn{3}{c}{\method{}} & Engaged & Final \\
  \cmidrule(lr){5-7}\cmidrule(lr){8-10}
  Dataset & Objective & & & Start & Peak & Final & Start & Peak & Final & (\%) & $\lambda_m$ \\
  \midrule
  \multirow{3}{*}{GSM8K}
   & Accuracy              & $1.44$  & $0.000$ & $0.080$ & $0.900$ & $0.848$ & $0.080$ & $0.910$ & $0.830$ & $76.8$ & $0.520$ \\
   & Clarity               & $0.51$  & $0.000$ & $0.060$ & $1.000$ & $1.000$ & $0.060$ & $1.000$ & $1.000$ & $0.0$ & $0.334$ \\
   & Conciseness$^\dagger$ & $0.84$  & $0.001$ & $1.000$ & $1.000$ & $0.884$ & $1.000$ & $1.000$ & $0.932$ & $10.8$ & $0.338$ \\
  \midrule
  \multirow{2}{*}{MBPP}
   & Pass rate             & $0.81$  & $0.000$ & $0.565$ & $0.565$ & $0.424$ & $0.565$ & $0.565$ & $0.475$ & $100.0$ & $1.005$ \\
   & Conciseness$^\dagger$ & $0.85$  & $0.000$ & $1.000$ & $1.000$ & $0.864$ & $1.000$ & $1.000$ & $0.720$ & $86.1$ & $0.594$ \\
  \midrule
  \multirow{3}{*}{\begin{tabular}{@{}l@{}}PKU-\\SafeRLHF\end{tabular}}
   & Harmlessness          & $0.00$  & $0.439$ & $0.999$ & $0.999$ & $0.999$ & $0.999$ & $0.999$ & $0.999$ & $100.0$ & $2.248$ \\
   & Helpfulness           & $\llap{$-$}0.01$  & $0.286$ & $0.390$ & $0.722$ & $0.717$ & $0.390$ & $0.749$ & $0.742$ & $99.7$ & $0.609$ \\
   & Conciseness$^\dagger$ & $1.11$  & $0.000$ & $1.000$ & $1.000$ & $1.000$ & $1.000$ & $1.000$ & $1.000$ & $0.2$ & $0.330$ \\
  \midrule
  \multirow{3}{*}{XSum}
   & Coverage              & $0.04$  & $0.019$ & $0.241$ & $0.241$ & $0.054$ & $0.246$ & $0.246$ & $0.053$ & $100.0$ & $1.573$ \\
   & Faithfulness          & $0.30$  & $0.001$ & $0.155$ & $0.869$ & $0.794$ & $0.158$ & $0.831$ & $0.737$ & $44.6$ & $0.378$ \\
   & Conciseness$^\dagger$ & $0.13$  & $0.000$ & $1.000$ & $1.000$ & $0.944$ & $1.000$ & $1.000$ & $0.961$ & $11.4$ & $0.348$ \\
  \bottomrule
  \end{tabular}%
  }
\end{table}

\paragraph{Covariance is large while an objective improves and small once it plateaus.} \Cref{eq:local-cov-law} predicts this pattern at first order, and the Linear training logs behind \Cref{tab:scope} show it. On GSM8K, test accuracy rises from 0.08 to 0.90 over the first 35 steps and then plateaus. Its covariance, in units of $10^{-2}$, averages 5.7 during the rise and 0.91 afterward. XSum faithfulness follows the same pattern. It rises early to about 0.45, then climbs again to 0.74 between steps 90 and 115 before plateauing near 0.80. Its covariance averages 0.84 during this second climb and 0.19 afterward. The exceptions are MBPP pass rate and XSum coverage, which decline even though their mean covariance is positive.

\paragraph{\method{} slows a gradual decline but cannot reverse a collapse.} On MBPP, the pass rate declines gradually from its first evaluation, and the controller raises its weight at every step, from 50\% to 63\% of the total by the end of training. \method{} slows this decline, reducing the drop from 0.141 to 0.090, at the cost of conciseness, which ends at 0.720 compared with 0.864 under Linear. XSum coverage, in contrast, loses nearly half its value within the first ten steps and collapses again near step 100 as summaries shrink to roughly 4 tokens. The controller raises the coverage weight at every step, but too slowly. Its share reaches only 35\% by step 10 and 40\% by step 100, when the two declines occur, climbing to 68\% only by the end. Coverage ends on par with Linear (0.053 vs.\ 0.054), and faithfulness ends lower (0.737 vs.\ 0.794). The other two domains show no interference, so there is nothing for the controller to counteract.

\begin{table}[ht]
\centering
\small
\setlength{\tabcolsep}{5pt}
\renewcommand{\arraystretch}{1.1}
\caption{\textbf{Effect of \method{}'s covariance target on accuracy at 7B} on MATH-500 using Qwen2.5-7B-Base. \emph{Peak}: best test accuracy up to step 210. \emph{Accuracy}: test accuracy at step 210. \emph{Drop}: Peak minus Accuracy. \emph{Final}: mean test accuracy over the last five evaluations of the full run. Unlike at 1.5B (\Cref{tab:drag-math}), where the highest target works best, only the 0.10 target reliably reduces the drop and stays above Linear through the full run. This suggests that a larger model may need a more conservative target to avoid destabilizing training.}
\label{tab:7b-sweep}
\begin{tabular}{@{}lcccc@{}}
\toprule
& \multicolumn{3}{c}{Steps 0-210} & Full run (315 steps) \\
\cmidrule(lr){2-4}\cmidrule(l){5-5}
Target $c^\ast_{\text{acc}}$ & Peak $\uparrow$ & Accuracy $\uparrow$ & Drop $\downarrow$ & Final $\uparrow$ \\
\midrule
$0.00$ (Linear) & $0.663$\sd{0.006} & $0.307$\sd{0.171} & $0.357$\sd{0.169} & $0.299$\sd{0.212} \\
$0.10$          & $0.643$\sd{0.029} & $0.457$\sd{0.025} & $0.187$\sd{0.051} & $0.449$\sd{0.049} \\
$0.20$          & $0.657$\sd{0.015} & $0.483$\sd{0.035} & $0.173$\sd{0.032} & $0.276$\sd{0.247} \\
$0.30$          & $0.653$\sd{0.015} & $0.433$\sd{0.178} & $0.220$\sd{0.165} & $0.453$\sd{0.118} \\
$0.40$          & $0.653$\sd{0.006} & $0.160$\sd{0.277} & $0.493$\sd{0.280} & $0.401$\sd{0.126} \\
\bottomrule
\end{tabular}
\end{table}

\begin{figure}[ht]
  \centering
  \setlength{\tabcolsep}{1pt}
  \renewcommand{\arraystretch}{0}

  \begin{subfigure}{\textwidth}
    \centering
    \begin{tabular}{ccc}
      \includegraphics[width=0.33\textwidth]{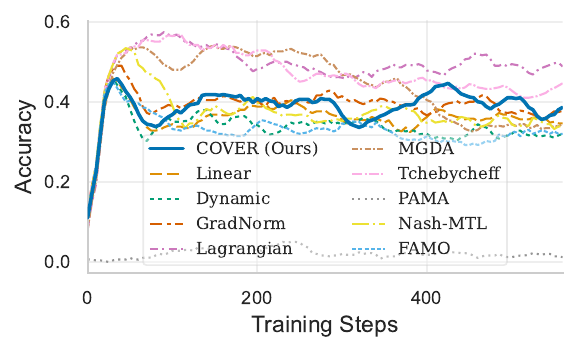} &
      \includegraphics[width=0.33\textwidth]{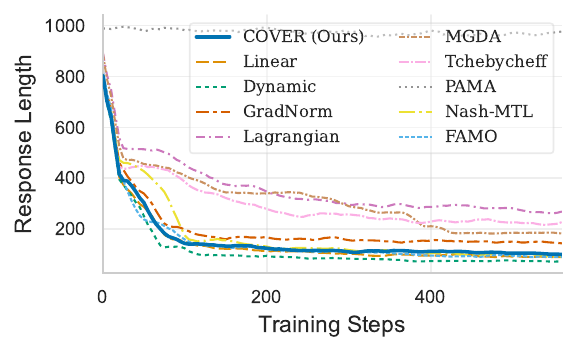} &
      \includegraphics[width=0.33\textwidth]{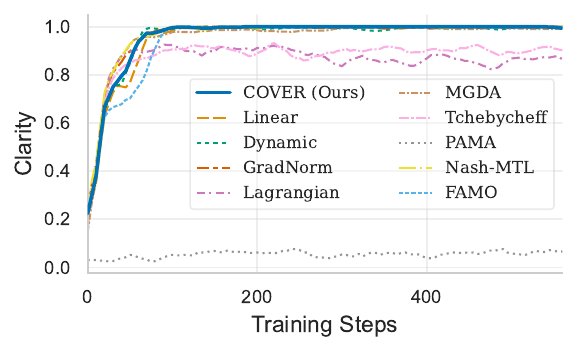}
    \end{tabular}
    \caption{Qwen2.5-1.5B-Base}
    \label{fig:grpo-qwen2.5-base}
  \end{subfigure}

  \vspace{-2pt}

  \begin{subfigure}{\textwidth}
    \centering
    \begin{tabular}{ccc}
      \includegraphics[width=0.33\textwidth]{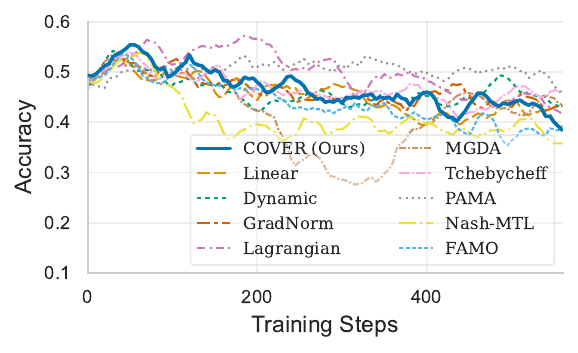} &
      \includegraphics[width=0.33\textwidth]{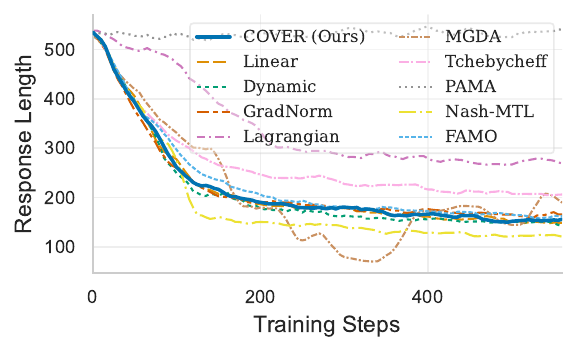} &
      \includegraphics[width=0.33\textwidth]{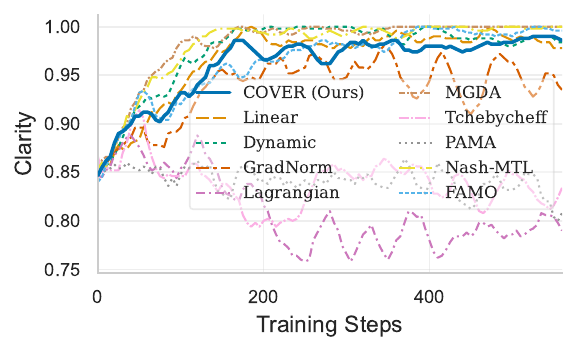}
    \end{tabular}
    \caption{Qwen2.5-1.5B-IFT}
    \label{fig:grpo-qwen2.5-ift}
  \end{subfigure}

  \vspace{-2pt}

  \begin{subfigure}{\textwidth}
    \centering
    \begin{tabular}{ccc}
      \includegraphics[width=0.33\textwidth]{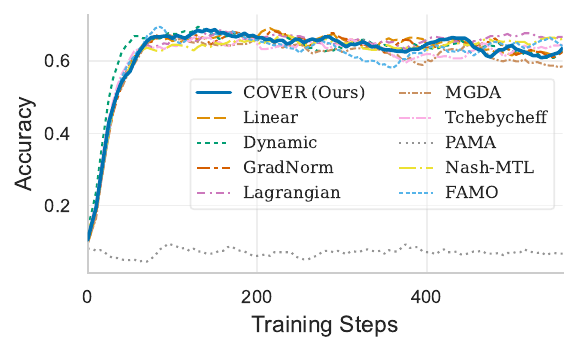} &
      \includegraphics[width=0.33\textwidth]{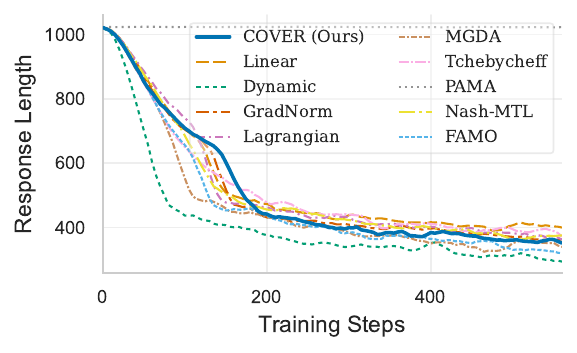} &
      \includegraphics[width=0.33\textwidth]{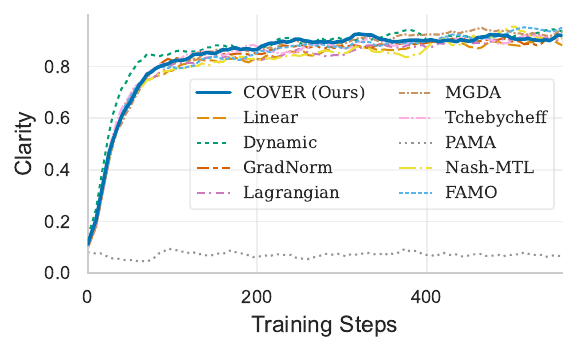}
    \end{tabular}
    \caption{Qwen3-1.7B-Base}
    \label{fig:grpo-qwen3-base}
  \end{subfigure}
  \caption{\textbf{MATH-500 test performance over GRPO training for Qwen2.5 and Qwen3 models.} Under \Cref{def:interference}, all ten methods interfere with accuracy on Qwen2.5-1.5B-Base (PAMA never learns the task), nine of ten on Qwen2.5-1.5B-IFT (all but PAMA, whose policy hardly changes), and one of ten on Qwen3-1.7B-Base. The pattern across models is the same under REINFORCE (\Cref{fig:main intro figure,fig:reinforce-ift-appendix}), so interference is not specific to the unclipped objective. \method{} matches Linear on every model. On the two Qwen2.5 models, Lagrangian and Tchebycheff keep higher accuracy at the cost of longer responses and lower clarity.}
  \label{fig:grpo result}
\end{figure}
\FloatBarrier

\begin{table}[ht]
\centering
\small
\setlength{\tabcolsep}{5pt}
\renewcommand{\arraystretch}{1.1}
\caption{\textbf{Effect of spectral decoupling on MATH-500} with Qwen2.5-1.5B-Base using REINFORCE. Spectral decoupling \citep{pezeshki2021gradient} adds the logit-norm penalty $\tfrac{\kappa}{2}\|\text{logits}\|^2$ to the policy loss as a remedy for gradient starvation. We test whether this remedy removes cross-objective interference under Linear, with accuracy-only training as a control. \emph{Gradient norm}: median gradient norm over the first five updates, measured before gradient-norm clipping to 1. \emph{Final}: mean of the last five test evaluations. \emph{Final entropy}: policy entropy on the training batch at the last update, whose maximum $\ln|\mathcal{X}|=\ln 151936=11.93$ is that of the uniform distribution over the vocabulary. Dashes mark unavailable quantities. Because the penalty sums over the full vocabulary, even $\kappa=10^{-4}$ raises the gradient norm to 1,072, about 2,500 times the 0.43 of the unpenalized accuracy-only run. Every penalized run reaches this maximum entropy, including the accuracy-only runs, which have no competing objective. The collapse therefore comes from the penalty itself, not from cross-objective interference. \textbf{At every tested strength, spectral decoupling collapses the policy.}}
\label{tab:sd-probe}
\begin{tabular}{@{}lcccccc@{}}
\toprule
& & & \multicolumn{3}{c}{Test accuracy} & \\
\cmidrule(lr){4-6}
Objectives & Penalty $\kappa$ & Gradient norm & Peak (step) & Step 25 & Final $\uparrow$ & Final entropy \\
\midrule
Accuracy only & $0$ & $0.43$ & $0.610$ (165) & $0.440$ & $0.526$ & $0.11$ \\
Accuracy only & $3\times10^{-4}$ & $3{,}160$ & $0.200$ (5) & $0.000$ & $0.000$ & $11.93$ \\
Accuracy only & $10^{-3}$ & $10{,}971$ & $0.190$ (5) & $0.020$ & $0.000$ & $11.93$ \\
Accuracy only & $3\times10^{-3}$ & $28{,}495$ & $0.220$ (10) & $0.000$ & $0.000$ & $11.93$ \\
Accuracy only & $10^{-2}$ & $109{,}646$ & $0.190$ (10) & $0.000$ & $0.000$ & $11.93$ \\
\midrule
All three (Linear) & $0$ & -- & $0.520$ (50) & $0.440$ & $0.370$ & -- \\
All three (Linear) & $10^{-4}$ & $1{,}072$ & $0.220$ (10) & $0.000$ & $0.000$ & $11.93$ \\
All three (Linear) & $3\times10^{-4}$ & $2{,}845$ & $0.200$ (10) & $0.020$ & $0.000$ & $11.93$ \\
All three (Linear) & $10^{-3}$ & $9{,}362$ & $0.210$ (15) & $0.000$ & $0.000$ & $11.93$ \\
All three (Linear) & $3\times10^{-3}$ & $32{,}148$ & $0.190$ (10) & $0.000$ & $0.000$ & $11.93$ \\
All three (Linear) & $10^{-2}$ & $107{,}128$ & $0.180$ (10) & $0.000$ & $0.000$ & $11.93$ \\
\bottomrule
\end{tabular}
\end{table}

\clearpage
\section{Limitations and Future Work}
\label{appendix:limitations}

\paragraph{Scope of \method{}.}
\method{} helps only when an objective declines gradually, and even then it protects that objective at the cost of conciseness rather than yielding a Pareto improvement (Appendix~\ref{appendix:scope}). When objectives co-improve or one is saturated, it performs on par with Linear, and when an objective collapses, raising its weight does not bring it back. Reversing a collapse likely requires more than reweighting, and we leave it to future work.

\paragraph{Scope of baselines.}
We select representative baselines that either have been shown to be effective for multi-objective LLM alignment, such as Dynamic, PAMA, and Linear, or are canonical methods in MTL and MOO, such as MGDA, GradNorm, and Nash-MTL. Our goal is not to exhaustively benchmark all existing methods from the MTL, MOO, and LLM alignment literatures, but rather to uncover the limitations of major reward-level and gradient-level scalarization paradigms for LLM alignment. We leave the adaptation and evaluation of additional approaches, such as SDMGrad \citep{NEURIPS2023_0e5b96f9} and FairGrad \citep{maheshwari2023fairgrad}, to future work.

\paragraph{Practical application.}
\Cref{theorem:pl for raw} suggests why cross-objective interference is model-dependent, but directly turning it into a practical training strategy is challenging, because Assumption~\ref{assumption:aligned-gradients} has token-level Jacobian terms that are prohibitively expensive to compute for modern LLMs. We also explored an approximate solution that projects per-token logit gradients onto a cone around a shared reference direction during backpropagation, but it did not yield the desired gains. We therefore view \cref{theorem:pl for raw} mainly as an explanatory framework, and we hope it may inspire more scalable training methods in future work.

\section{Formal Analysis for Clipped Surrogate Objectives}
\label{appendix:analysis}
We establish useful lemmas and theorems that analyze sufficient conditions for multi-objective improvement under clipped surrogate objectives. We use GRPO as a running example, while the same argument applies to PPO-style \citep{schulman2017proximalpolicyoptimizationalgorithms} clipped surrogate objectives, with the only difference being the advantage estimator. All proofs are deferred to Appendix~\ref{appendix:proof}.

Fix a prompt $\vect{x}$ and sample a group of $K$ completions $\vect{y}^{(1)},\dots,\vect{y}^{(K)}\sim p_{\theta_{\mathrm{old}}}(\cdot\mid \vect{x})$.
Let $s_j \coloneq s(\vect{x},\vect{y}^{(j)})$ and define the group normalization $A_i(\vect{x}) \coloneq \frac{s_i-\bar{s}(\vect{x})}{\hat{\sigma}(\vect{x})}$. The GRPO surrogate objective is
\begin{align*}
J(\theta)\coloneq
\E\Bigg[
\sum_{k=1}^{K}\sum_{l=1}^{L_{\mathrm{out}}}
\min\Big\{\rho_{k,l}(\theta)A_k(\vect{x}),\,\bar{\rho}_{k,l}(\theta)A_k(\vect{x})\Big\}
\Bigg]
-\beta\,\E\big[\mathrm{KL}(p_{\theta}\|p_{\mathrm{ref}})\big]
+\lambda\,\E\big[H(p_{\theta})\big],
\end{align*}
where $\beta,\lambda\ge 0$ control KL and entropy regularization. $\bar{\rho}_{k,l}(\theta)$ is the clipped result of the importance ratio $\rho_{k,l}(\theta)$, $\bar{\rho}_{k,l}(\theta)=\mathrm{clip}\big(\rho_{k,l}(\theta),1-\varepsilon,1+\varepsilon\big)$. We know that the unclipped indicator selecting the active branch of the minimum is
\begin{equation}
\label{eq:unclipped-indicator}
\mathbf{1}_{k,l}(\theta)=
\begin{cases}
1, & A_k(\vect{x})\ge 0 \ \text{and}\ \rho_{k,l}(\theta)\le 1+\varepsilon,\\
1, & A_k(\vect{x})<0 \ \text{and}\ \rho_{k,l}(\theta)\ge 1-\varepsilon,\\
0, & \text{otherwise}.
\end{cases}
\end{equation}
Recall $f(\vect{x},\vect{y}_{\le l-1};\theta)\in\R^{|\mathcal{X}|}$ is the next-token logit map and $p_\theta(\cdot\mid \vect{x},\vect{y}_{\le l-1})=\mathrm{softmax}(f(\vect{x},\vect{y}_{\le l-1};\theta))$.
Define the tokenwise logit-gradient feature $\phi_l(\vect{x},\vect{y};\theta)$ and the clipped advantage weight $W_{k,l}(\theta)$:
\begin{equation}
\label{eq:Wkl}
\phi_l(\vect{x},\vect{y};\theta)\coloneq\nabla_{\theta}\log p_\theta\big(y_l\mid \vect{x},\vect{y}_{\le l-1}\big),\,
W_{k,l}(\theta)\coloneq A_k(\vect{x})\,\rho_{k,l}(\theta)\,\mathbf{1}_{k,l}(\theta).
\end{equation}

\begin{lemma}
\label{lemma:universal-lower-bound}
Let $r_m(\theta)\coloneq r_m(p_\theta)$ and assume $r_m$ has an $L_m$-Lipschitz gradient around $\theta$.
For any direction $d(\theta) = \nabla J(\theta)$ and update $\theta^{+}=\theta+\eta d(\theta)$, we have
\begin{equation*}
r_m(\theta^{+})-r_m(\theta)
\ \ge\
\eta\,\langle \nabla r_m(\theta), d(\theta)\rangle
-\frac{L_m}{2}\eta^2\|d(\theta)\|^2,\; \forall m.
\end{equation*}
In particular, if $\langle \nabla r_m(\theta), d(\theta)\rangle\ge 0$ and $\eta>0$ is sufficiently small,
then $r_m(\theta^{+})\ge r_m(\theta)$.
\end{lemma}

\begin{theorem}[Fisher-covariance sufficient condition for natural gradient updates]
\label{theorem:fisher-cov-sufficient}
For completions $\vect{y}^{(1)},\dots,\vect{y}^{(K)}\sim p_{\theta_{\mathrm{old}}}(\cdot\mid\vect{x})$, write $\phi_{k,l}(\theta)\coloneq \phi_l(\vect{x},\vect{y}^{(k)};\theta)$.
Define the aggregated Fisher matrix:
$$
F(\theta)\coloneq 
\E\Bigg[\Big(\sum_{k=1}^{K}\sum_{l=1}^{L_{\mathrm{out}}} \phi_{k,l}(\theta)\Big)
\Big(\sum_{k=1}^{K}\sum_{l=1}^{L_{\mathrm{out}}} \phi_{k,l}(\theta)\Big)^\top\Bigg],
$$
and the weighted feature mean and regularizer gradient:
\begin{align*}
G(\theta)\coloneq 
\E\Bigg[\sum_{k=1}^{K}\sum_{l=1}^{L_{\mathrm{out}}} W_{k,l}(\theta)\,\phi_{k,l}(\theta)\Bigg], \; &R(\theta)\coloneq \beta\,\nabla_{\theta}\E\big[\mathrm{KL}(p_\theta\|p_{\mathrm{ref}})\big]
-\lambda\,\nabla_{\theta}\E\big[H(p_\theta)\big].
\end{align*}
Assume $F(\theta)$ is invertible and consider the natural gradient direction $d_{\mathrm{nat}}(\theta)\coloneq F(\theta)^{-1}\big(G(\theta)-R(\theta)\big),\;
\theta^{+}=\theta+\eta\, d_{\mathrm{nat}}(\theta)$. If for every $m$,
\begin{equation}
\label{eq:fisher-cov-condition}
\nabla r_m(\theta)^{\top}F(\theta)^{-1}G(\theta)
\ \ge\
\nabla r_m(\theta)^{\top}F(\theta)^{-1}R(\theta),
\end{equation}
and moreover either $d_{\mathrm{nat}}(\theta)=0$ or the inequality holds with a positive margin
$$
\gamma_m(\theta)\coloneq \nabla r_m(\theta)^{\top}F(\theta)^{-1}\big(G(\theta)-R(\theta)\big)>0,\;\forall m,
$$
then for sufficiently small $\eta>0$ we have $r_m(\theta^{+})\ge r_m(\theta)$ for all $m$ and strictly if $\min_m \gamma_m(\theta)>0$.
\end{theorem}

\begin{corollary}[Categorical bandit case]
\label{cor:categorical-reduction}
Fix a prompt $\vect{x}$ and suppose the policy over completions $\vect{y}\in\mathcal{X}^{L_{\mathrm{out}}}$ is a categorical distribution parameterized by logits $\theta\in\R^{|\mathcal{X}^{L_{\mathrm{out}}}|}$
$$
p_\theta(\vect{y}\mid \vect{x})=\frac{\exp(\theta_{\vect{y}})}{\sum_{\vect{y}'}\exp(\theta_{\vect{y}'})}.
$$
Let $w(\vect{x},\vect{y})$ be an arbitrary scalar weight assigned to each completion (i.e., a per-sample quantity whose expected value $\E_{\vect{y}\sim p_\theta(\cdot\mid\vect{x})}[w(\vect{x},\vect{y})]$ we seek to maximize via natural gradient ascent).
Define the categorical Fisher matrix
\begin{align*}
F(\theta)&\coloneq
\E_{\vect{y}\sim p_\theta(\cdot\mid\vect{x})}\big[\nabla_\theta \log p_\theta(\vect{y}\mid\vect{x})
\nabla_\theta \log p_\theta(\vect{y}\mid\vect{x})^{\top}\big].
\end{align*}
Let $d_{\mathrm{nat}}(\theta)$ be any natural gradient direction and take the update $\theta^{+}=\theta+\eta d_{\mathrm{nat}}(\theta)$ with learning rate $\eta>0$. Then for each objective $m$, we have
\begin{align*}
\E_{\vect{y}\sim p_{\theta^{+}}(\cdot\mid\vect{x})}\big[r_m(\vect{x},\vect{y})\big]-\E_{\vect{y}\sim p_{\theta}(\cdot\mid\vect{x})}\big[r_m(\vect{x},\vect{y})\big] =\eta\mathrm{Cov}_{\vect{y}\sim p_{\theta}(\cdot\mid \vect{x})}\big(r_m(\vect{x},\vect{y}),w(\vect{x},\vect{y})\big)+O(\eta^2).
\end{align*}
Specifically, $\mathrm{Cov}\big(r_m(\vect{x},\vect{y}),w(\vect{x},\vect{y})\big)\ge 0$ for all $m$ is a sufficient first-order condition to ensure that no objective degrades.
\end{corollary}

\begin{corollary}[Clipping robustness]
\label{cor:clipping-robustness}
Define the unclipped and clipped weights for each token
$$
w^{\mathrm{unclip}}_{k,l}(\theta)\coloneq A_k(\vect{x})\rho_{k,l}(\theta),
w^{\mathrm{clip}}_{k,l}(\theta)\coloneq w^{\mathrm{unclip}}_{k,l}(\theta)\mathbf{1}_{k,l}(\theta),
$$
where $\mathbf{1}_{k,l}(\theta)\in\{0,1\}$ is defined in \Cref{eq:unclipped-indicator} and $\mathbf{1}_{k,l}(\theta)=0$ exactly when the term is clipped away and contributes zero gradient.
Let
\begin{align*}
G^{\mathrm{unclip}}(\theta)\coloneq \E\Bigg[\sum_{k=1}^{K}\sum_{l=1}^{L_{\mathrm{out}}} 
w^{\mathrm{unclip}}_{k,l}(\theta)\phi_{k,l}(\theta)\Bigg],\;
G^{\mathrm{clip}}(\theta)\coloneq \E\Bigg[\sum_{k=1}^{K}\sum_{l=1}^{L_{\mathrm{out}}} 
w^{\mathrm{clip}}_{k,l}(\theta)\phi_{k,l}(\theta)\Bigg].
\end{align*}
For each objective $m$, define the \emph{unclipped} first-order margin
$$
\gamma_m^{\mathrm{unclip}}(\theta)\coloneq \nabla r_m(\theta)^{\top}F(\theta)^{-1}\big(G^{\mathrm{unclip}}(\theta)-R(\theta)\big).
$$
Then the clipped first-order margin satisfies
\begin{align*}
\nabla r_m(\theta)^{\top}F(\theta)^{-1}&\big(G^{\mathrm{clip}}(\theta)-R(\theta)\big) \\
&\ge 
\gamma_m^{\mathrm{unclip}}(\theta) - \big\|F(\theta)^{-1/2}\nabla r_m(\theta)\big\|\cdot
\big\|F(\theta)^{-1/2}\big(G^{\mathrm{unclip}}(\theta)-G^{\mathrm{clip}}(\theta)\big)\big\|.
\end{align*}
Consequently, if $\gamma_m^{\mathrm{unclip}}(\theta)\ge \kappa_m>0$ and
\begin{equation*}
\big\|F(\theta)^{-1/2}\big(G^{\mathrm{unclip}}(\theta)-G^{\mathrm{clip}}(\theta)\big)\big\| \le
\frac{\kappa_m}{\big\|F(\theta)^{-1/2}\nabla r_m(\theta)\big\|},
\end{equation*}
then
$\nabla r_m(\theta)^{\top}F(\theta)^{-1}\big(G^{\mathrm{clip}}(\theta)-R(\theta)\big)\ge 0$,
so the sufficient condition for $r_m(\theta^{+})\ge r_m(\theta)$ in \Cref{theorem:fisher-cov-sufficient} still holds for objective $m$ under clipping.
\end{corollary}

\begin{remark}
Corollary \ref{cor:clipping-robustness} shows that clipping can only affect a first-order improvement guarantee by removing a subset of weighted logit-gradient terms $w^{\mathrm{unclip}}_{k,l}(\theta)\,\phi_{k,l}(\theta)$ from the update. The Fisher distortion $\|F(\theta)^{-1/2}(G^{\mathrm{unclip}}(\theta)-G^{\mathrm{clip}}(\theta))\|$ therefore measures, in natural gradient geometry, how much ``gradient mass'' is deleted by clipping. This distortion becomes large when many importance ratios fall outside the clipping window, especially on samples with large magnitude advantages. Importantly, this can be carefully controlled in practice through techniques such as learning rate scheduling or reward normalization, ensuring that clipping distortion remains small and the covariance law continues to hold.
\end{remark}

\section{Proofs}
\label{appendix:proof}
\subsection{Proof of Lemma \ref{lemma:exp-tilt}}
\begin{lemma}
\label{lemma:exp-tilt}
The optimizer of \Cref{eq:kl-policy-improvement} is
$
p_{\theta;\vect{x}}^{+}(\vect{y})=\frac{p_{\theta;\vect{x}}(\vect{y})\exp\!\big(\eta\,s(\vect{x},\vect{y})\big)}{\E_{\vect{y}\sim p_{\theta;\vect{x}}}\!\Big[\exp\!\big(\eta\,s(\vect{x},\vect{y})\big)\Big]}.
$
\end{lemma}
\begin{proof}
Fix $\vect{x}$. We optimize over distributions $q$ on $\mathcal{X}^{L_{\mathrm{out}}}$ satisfying
$q(\vect{y})\ge 0$ and $\sum_{\vect{y}}q(\vect{y})=1$.
Expanding the KL term gives
$$
\mathrm{KL}(q\|p_{\theta;\vect{x}})=\sum_{\vect{y}}q(\vect{y})\log\frac{q(\vect{y})}{p_{\theta;\vect{x}}(\vect{y})}.
$$
Thus the objective can be written as
$$
\max_{q \in \Delta(\mathcal{X}^{L_{\mathrm{out}}})}
\left\{
\sum_{\vect{y}}q(\vect{y})s(\vect{x},\vect{y})
-\frac{1}{\eta}\sum_{\vect{y}}q(\vect{y})\log\frac{q(\vect{y})}{p_{\theta;\vect{x}}(\vect{y})}
\right\}.
$$
Introduce a Lagrange multiplier $\lambda\in\R$ for the normalization constraint $\sum_{\vect{y}} q(\vect{y})=1$, and consider
$$
\mathcal{L}(q,\lambda)=\sum_{\vect{y}}q(\vect{y})\,s(\vect{x},\vect{y})-\frac{1}{\eta}\sum_{\vect{y}}q(\vect{y})\log\frac{q(\vect{y})}{p_{\theta;\vect{x}}(\vect{y})}+\lambda\Big(\sum_{\vect{y}}q(\vect{y})-1\Big).
$$
For any $\vect{y}$ with $q(\vect{y})>0$, the stationarity condition is
\begin{align*}
\frac{\partial}{\partial q(\vect{y})}\mathcal{L}(q,\lambda)=s(\vect{x},\vect{y})-\frac{1}{\eta}\left(\log\frac{q(\vect{y})}{p_{\theta;\vect{x}}(\vect{y})}+1\right)+\lambda=0.
\end{align*}
Rearranging yields
$$
\log\frac{q(\vect{y})}{p_{\theta;\vect{x}}(\vect{y})}
=
\eta\,s(\vect{x},\vect{y})+\eta\lambda-1.
$$
Exponentiating both sides gives us
$$
q(\vect{y})=p_{\theta;\vect{x}}(\vect{y})\exp\!\big(\eta\,s(\vect{x},\vect{y})\big)\exp(\eta\lambda-1).
$$
Let $C:=\exp(\eta\lambda-1)$, which does not depend on $\vect{y}$. Enforcing normalization $\sum_{\vect{y}} q(\vect{y})=1$ implies
$$
1=\sum_{\vect{y}}q(\vect{y})=C\sum_{\vect{y}}p_{\theta;\vect{x}}(\vect{y})\exp\!\big(\eta\,s(\vect{x},\vect{y})\big),
$$
so
$$
C=\frac{1}{\sum_{\vect{y}}p_{\theta;\vect{x}}(\vect{y})\exp\!\big(\eta\,s(\vect{x},\vect{y})\big)}=\frac{1}{\E_{\vect{y}\sim p_{\theta;\vect{x}}}\big[\exp(\eta\,s(\vect{x},\vect{y}))\big]}.
$$
Substituting back, the maximizer is
$$
q(\vect{y})
=
\frac{p_{\theta;\vect{x}}(\vect{y})\exp\!\big(\eta\,s(\vect{x},\vect{y})\big)}
{\E_{\vect{y}\sim p_{\theta;\vect{x}}}\big[\exp(\eta\,s(\vect{x},\vect{y}))\big]}.
$$
Identifying $q=p_{\theta;\vect{x}}^{+}$ yields the claim.
\end{proof}
\begin{remark}
The update in \Cref{eq:kl-policy-improvement} trades off increasing $\E_{\vect{y}\sim q}[s(\vect{x},\vect{y})]$ with staying close to $p_{\theta;\vect{x}}$ in $\mathrm{KL}(q\|p_{\theta;\vect{x}})$.
Lemma \ref{lemma:exp-tilt} shows the solution is an exponential tilting,
$
p_{\theta;\vect{x}}^{+}(\vect{y}) \propto p_{\theta;\vect{x}}(\vect{y})e^{\eta s(\vect{x},\vect{y})},
$
so it can only reweight completions already supported by $p_{\theta;\vect{x}}$.
In particular, if $p_{\theta;\vect{x}}(\vect{y})=0$ then $p_{\theta;\vect{x}}^{+}(\vect{y})=0$, since otherwise $\mathrm{KL}(q\|p_{\theta;\vect{x}})=+\infty$.
\end{remark}

\subsection{Proof of \Cref{theorem:local-cov-law}}
\begin{proof}
Fix $\vect{x}$. By Lemma~\ref{lemma:exp-tilt}, the optimizer of \Cref{eq:kl-policy-improvement} satisfies
$$
p_{\theta;\vect{x}}^{+}(\vect{y})=\frac{p_{\theta;\vect{x}}(\vect{y})\exp\big(\eta\,s(\vect{x},\vect{y})\big)}{\E_{\vect{y}\sim p_{\theta;\vect{x}}}\Big[\exp\big(\eta\,s(\vect{x},\vect{y})\big)\Big]}.
$$
Let $\bar{s}(\vect{x})\coloneq \E_{\vect{y}\sim p_{\theta;\vect{x}}}\big[s(\vect{x},\vect{y})\big]$.
Using the Taylor expansion $\exp(\eta s)=1+\eta s+O(\eta^2)$,
we have
$$
Z_\eta(\vect{x})=\E_{\vect{y}\sim p_{\theta;\vect{x}}}\big[1+\eta s(\vect{x},\vect{y})+O(\eta^2)\big]=1+\eta\bar{s}(\vect{x})+O(\eta^2).
$$
For small $\eta$, $(1+u)^{-1}=1-u+O(u^2)$ implies
$$
\frac{1}{Z_\eta(\vect{x})}
=
1-\eta\,\bar{s}(\vect{x})+O(\eta^2).
$$
Substituting back yields the pointwise expansion
$$
p_{\theta;\vect{x}}^{+}(\vect{y})=p_{\theta;\vect{x}}(\vect{y})\Big(1+\eta\big(s(\vect{x},\vect{y})-\bar{s}(\vect{x})\big)\Big)+O(\eta^2)p_{\theta;\vect{x}}(\vect{y}).
$$
Therefore,
\begin{align*}
&\E_{\vect{y}\sim p_{\theta;\vect{x}}^{+}}\big[r_m(\vect{x},\vect{y})\big]
=
\sum_{\vect{y}} p_{\theta;\vect{x}}^{+}(\vect{y})r_m(\vect{x},\vect{y})\\
&=
\sum_{\vect{y}} p_{\theta;\vect{x}}(\vect{y})r_m(\vect{x},\vect{y})+\eta\sum_{\vect{y}} p_{\theta;\vect{x}}(\vect{y})r_m(\vect{x},\vect{y})\big(s(\vect{x},\vect{y})-\bar{s}(\vect{x})\big)+O(\eta^2)\sum_{\vect{y}}p_{\theta;\vect{x}}(\vect{y})r_m(\vect{x},\vect{y})\\
&=
\E_{\vect{y}\sim p_{\theta;\vect{x}}}\big[r_m(\vect{x},\vect{y})\big]+\eta\Big(
\E_{\vect{y}\sim p_{\theta;\vect{x}}}\big[r_m(\vect{x},\vect{y})s(\vect{x},\vect{y})\big]-\E_{\vect{y}\sim p_{\theta;\vect{x}}}\big[r_m(\vect{x},\vect{y})\big]
\E_{\vect{y}\sim p_{\theta;\vect{x}}}\big[s(\vect{x},\vect{y})\big]
\Big)+O(\eta^2)\\
&=
\E_{\vect{y}\sim p_{\theta;\vect{x}}}\big[r_m(\vect{x},\vect{y})\big]+\eta
\mathrm{Cov}_{\vect{y}\sim p_{\theta}(\cdot\mid \vect{x})}
\big(r_m(\vect{x},\vect{y}),s(\vect{x},\vect{y})\big)+O(\eta^2).
\end{align*}
The third equality follows since $r_m(\vect{x},\vect{y})$ is bounded, so the remainder term is $O(\eta^2)$. Finally, taking expectation over $\vect{x}\sim \mathcal{D}$ gives \Cref{eq:local-cov-law}.
\end{proof}

\subsection{Proof of Lemma \ref{lemma:universal-lower-bound}}
\begin{proof}
Define the deterministic one-dimensional function
$$
g(t)\coloneq r_m\big(\theta+t d(\theta)\big),\; t\in[0,\eta].
$$
By the fundamental theorem of calculus,
\begin{align*}
r_m(\theta^{+})-r_m(\theta)
&= g(\eta)-g(0)
= \int_{0}^{\eta} g'(t)dt \\
&= \int_{0}^{\eta} \big\langle d(\theta),\nabla r_m\big(\theta+t d(\theta)\big)\big\rangle dt \\
&= \int_{0}^{\eta} \big\langle d(\theta),\nabla r_m(\theta)\big\rangle dt
+ \int_{0}^{\eta} \big\langle d(\theta),\nabla r_m\big(\theta+t d(\theta)\big)-\nabla r_m(\theta)\big\rangle dt \\
&= \eta\big\langle \nabla r_m(\theta), d(\theta)\big\rangle
+ \int_{0}^{\eta} \big\langle \nabla r_m\big(\theta+t d(\theta)\big)-\nabla r_m(\theta), d(\theta)\big\rangle dt .
\end{align*}
For the second term, the Cauchy--Schwarz inequality implies
$$
\big\langle \nabla r_m\big(\theta+t d(\theta)\big)-\nabla r_m(\theta), d(\theta)\big\rangle
\ge
-\|\nabla r_m\big(\theta+t d(\theta)\big)-\nabla r_m(\theta)\|\|d(\theta)\|.
$$
By the $L_m$-Lipschitzness of $\nabla r_m$,
$$
\|\nabla r_m\big(\theta+t d(\theta)\big)-\nabla r_m(\theta)\|
\le
L_m\|\theta+t d(\theta)-\theta\|
=
L_m t\|d(\theta)\|.
$$
Therefore,
\begin{align*}
\int_{0}^{\eta} \big\langle \nabla r_m\big(\theta+t d(\theta)\big)-\nabla r_m(\theta), d(\theta)\big\rangle dt
\ge
-\int_{0}^{\eta} L_m t\|d(\theta)\|^2 dt =
-\frac{L_m}{2}\eta^2\|d(\theta)\|^2.
\end{align*}
Combining the bounds yields
$$
r_m(\theta^{+})-r_m(\theta)
\ge
\eta\big\langle \nabla r_m(\theta), d(\theta)\big\rangle
-\frac{L_m}{2}\eta^2\|d(\theta)\|^2,
$$
which proves the lemma.
\end{proof}

\subsection{Proof of \Cref{theorem:fisher-cov-sufficient}}

\begin{proof}
Write the GRPO surrogate objective as
$$
J(\theta)=
\E\Bigg[
\sum_{k=1}^{K}\sum_{l=1}^{L_{\mathrm{out}}}
\min\Big\{\rho_{k,l}(\theta)A_k(\vect{x}),\bar{\rho}_{k,l}(\theta)A_k(\vect{x})\Big\}
\Bigg]
-\beta\E\big[\mathrm{KL}(p_{\theta}\|p_{\mathrm{ref}})\big]
+\lambda\E\big[H(p_{\theta})\big].
$$
Fix $(\vect{x},\vect{y}^{(1:K)})$ and a token index $(k,l)$. Define
$$
s_{k,l}(\theta)\coloneq \min\Big\{\rho_{k,l}(\theta)A_k(\vect{x}),\bar{\rho}_{k,l}(\theta)A_k(\vect{x})\Big\}.
$$
By the definition of tokenwise clipping, $s_{k,l}(\theta)$ equals $\rho_{k,l}(\theta)A_k(\vect{x})$
on the unclipped region and equals $\bar{\rho}_{k,l}(\theta)A_k(\vect{x})$ on the clipped region.
On the clipped region, $\bar{\rho}_{k,l}(\theta)$ is constant in $\theta$ (equal to $1+\varepsilon$ or $1-\varepsilon$),
so its gradient is zero. On the unclipped region,
$$
\nabla s_{k,l}(\theta)
=
A_k(\vect{x})\nabla\rho_{k,l}(\theta)
=
A_k(\vect{x})\rho_{k,l}(\theta)\nabla\log p_\theta\big(y_l^{(k)}\mid \vect{x},\vect{y}_{\le l-1}^{(k)}\big)
=
W_{k,l}(\theta)\phi_{k,l}(\theta),
$$
where $W_{k,l}(\theta)=A_k(\vect{x})\rho_{k,l}(\theta)\mathbf{1}_{k,l}(\theta)$ encodes exactly the unclipped region.
Therefore,
$$
\nabla_\theta \E\Big[\sum_{k,l} s_{k,l}(\theta)\Big]
=
\E\Big[\sum_{k,l} \nabla s_{k,l}(\theta)\Big]
=
\E\Big[\sum_{k,l} W_{k,l}(\theta)\phi_{k,l}(\theta)\Big]
=
G(\theta).
$$
For the regularizers,
$$
\nabla_\theta\Big(-\beta\E[\mathrm{KL}(p_{\theta}\|p_{\mathrm{ref}})]
+\lambda\E[H(p_{\theta})]\Big)
=
-\Big(\beta\nabla_\theta\E[\mathrm{KL}(p_{\theta}\|p_{\mathrm{ref}})]
-\lambda\nabla_\theta\E[H(p_{\theta})]\Big)
=
-R(\theta).
$$
Combining these two parts gives $\nabla J(\theta)=G(\theta)-R(\theta)$, hence
$$
d_{\mathrm{nat}}(\theta)=F(\theta)^{-1}\nabla J(\theta)=F(\theta)^{-1}\big(G(\theta)-R(\theta)\big).
$$

By Lemma \ref{lemma:universal-lower-bound}, for each objective $m$ and any direction $d(\theta)$,
$$
r_m(\theta+\eta d(\theta)) - r_m(\theta)
\ge
\eta\nabla r_m(\theta)^{\top} d(\theta)
-\frac{L_m}{2}\eta^2\|d(\theta)\|^2.
$$
Setting $d(\theta)=d_{\mathrm{nat}}(\theta)$ and using the definition of $d_{\mathrm{nat}}(\theta)$ yields
\begin{align*}
\nabla r_m(\theta)^{\top} d_{\mathrm{nat}}(\theta)
&=
\nabla r_m(\theta)^{\top}F(\theta)^{-1}\big(G(\theta)-R(\theta)\big) \\
&=
\nabla r_m(\theta)^{\top}F(\theta)^{-1}G(\theta)
-\nabla r_m(\theta)^{\top}F(\theta)^{-1}R(\theta).
\end{align*}
Thus the condition in \Cref{eq:fisher-cov-condition} implies $\nabla r_m(\theta)^{\top} d_{\mathrm{nat}}(\theta)\ge 0$ for all $m$.
If $d_{\mathrm{nat}}(\theta)=0$, then $\theta^{+}=\theta$ and $r_m(\theta^{+})=r_m(\theta)$ trivially.
Otherwise, if $\gamma_m(\theta)=\nabla r_m(\theta)^{\top} d_{\mathrm{nat}}(\theta)>0$ for all $m$, then
$$
r_m(\theta^{+})-r_m(\theta)
\ge
\eta\gamma_m(\theta)-\frac{L_m}{2}\eta^2\|d_{\mathrm{nat}}(\theta)\|^2.
$$
Choosing
$$
0<\eta \le \min_{m}\frac{2\gamma_m(\theta)}{L_m\|d_{\mathrm{nat}}(\theta)\|^2}
$$
ensures the right-hand side is non-negative for every $m$, hence $r_m(\theta^{+})\ge r_m(\theta)$ for all $m$.
If $\min_m\gamma_m(\theta)>0$, then taking $\eta$ strictly smaller than the bound gives strict improvement for all $m$.
\end{proof}

\subsection{Proof of Corollary \ref{cor:categorical-reduction}}
\begin{proof}
Fix $\vect{x}$. For categorical softmax logits, for any $\vect{y},\vect{z}\in\mathcal{X}^{L_{\mathrm{out}}}$,
the Jacobian satisfies
$$
\frac{\partial p_\theta(\vect{y}\mid\vect{x})}{\partial \theta_{\vect{z}}}
=
p_\theta(\vect{y}\mid\vect{x})\big(\mathbf{1}\{\vect{y}=\vect{z}\}-p_\theta(\vect{z}\mid\vect{x})\big).
$$

For each $\vect{z}$, differentiating the expectation yields
\begin{align*}
\frac{\partial}{\partial \theta_{\vect{z}}}\E_{\vect{y}\sim p_\theta(\cdot\mid\vect{x})}\big[w(\vect{x},\vect{y})\big]
&=
\sum_{\vect{y}} w(\vect{x},\vect{y})\frac{\partial p_\theta(\vect{y}\mid\vect{x})}{\partial \theta_{\vect{z}}} \\
&=
\sum_{\vect{y}} w(\vect{x},\vect{y})p_\theta(\vect{y}\mid\vect{x})
\big(\mathbf{1}\{\vect{y}=\vect{z}\}-p_\theta(\vect{z}\mid\vect{x})\big) \\
&=
p_\theta(\vect{z}\mid\vect{x})w(\vect{x},\vect{z})
-
p_\theta(\vect{z}\mid\vect{x})\sum_{\vect{y}}p_\theta(\vect{y}\mid\vect{x})w(\vect{x},\vect{y}) \\
&=
p_\theta(\vect{z}\mid\vect{x})
\Big(w(\vect{x},\vect{z})-\E_{\vect{y}\sim p_\theta(\cdot\mid\vect{x})}\big[w(\vect{x},\vect{y})\big]\Big).
\end{align*}
Therefore, we obtain its vector form
\begin{equation}
\label{eq:gradient-w}
\nabla_\theta \E_{\vect{y}\sim p_\theta(\cdot\mid\vect{x})}\big[w(\vect{x},\vect{y})\big]
=
p_\theta(\cdot\mid\vect{x})\odot
\Big(w(\vect{x},\cdot)-\E_{\vect{y}\sim p_\theta(\cdot\mid\vect{x})}\big[w(\vect{x},\vect{y})\big]\mathbf{1}\Big),
\end{equation}
where $\odot$ denotes the elementwise product. For the categorical softmax family, the Fisher matrix is
$$
F(\theta)=\mathrm{diag}\big(p_\theta(\cdot\mid\vect{x})\big)-p_\theta(\cdot\mid\vect{x})p_\theta(\cdot\mid\vect{x})^{\top}.
$$
For any vector $v$ indexed by $\vect{y}$, the $\vect{y}$-th coordinate of $F(\theta)v$ is
\begin{align*}
\big[F(\theta)v\big]_{\vect{y}}
&=
\big[\mathrm{diag}(p_\theta(\cdot\mid\vect{x}))v\big]_{\vect{y}}
-\big[p_\theta(\cdot\mid\vect{x})p_\theta(\cdot\mid\vect{x})^{\top}v\big]_{\vect{y}} \\
&=
p_\theta(\vect{y}\mid\vect{x})v_{\vect{y}}
-
p_\theta(\vect{y}\mid\vect{x})\sum_{\vect{z}}p_\theta(\vect{z}\mid\vect{x})v_{\vect{z}} \\
&=
p_\theta(\vect{y}\mid\vect{x})
\Big(v_{\vect{y}}-\E_{\vect{z}\sim p_\theta(\cdot\mid\vect{x})}[v_{\vect{z}}]\Big).
\end{align*}
Now take
$$
v \;=\; w(\vect{x},\cdot)
-\E_{\vect{y}\sim p_\theta(\cdot\mid\vect{x})}\big[w(\vect{x},\vect{y})\big]\mathbf{1}.
$$
Then
\begin{align*}
\E_{\vect{z}\sim p_\theta(\cdot\mid\vect{x})}[v_{\vect{z}}]
&=
\E_{\vect{z}\sim p_\theta(\cdot\mid\vect{x})}\big[w(\vect{x},\vect{z})\big]
-
\E_{\vect{y}\sim p_\theta(\cdot\mid\vect{x})}\big[w(\vect{x},\vect{y})\big]\cdot
\E_{\vect{z}\sim p_\theta(\cdot\mid\vect{x})}[1] \\
&=0,
\end{align*}
so the general coordinate formula reduces to
$$
\big[F(\theta)v\big]_{\vect{y}}=p_\theta(\vect{y}\mid\vect{x})v_{\vect{y}}.
$$
Equivalently,
$$
F(\theta)v
=
p_\theta(\cdot\mid\vect{x})\odot v
=
p_\theta(\cdot\mid\vect{x})\odot
\Big(w(\vect{x},\cdot)-\E_{\vect{y}\sim p_\theta(\cdot\mid\vect{x})}\big[w(\vect{x},\vect{y})\big]\mathbf{1}\Big).
$$
Comparing with the expression for $\nabla_\theta \E[w]$ from \Cref{eq:gradient-w} shows that this particular $v$ satisfies
$$
F(\theta)v
=
\nabla_\theta \E_{\vect{y}\sim p_\theta(\cdot\mid\vect{x})}\big[w(\vect{x},\vect{y})\big].
$$
By definition, any vector $d_{\mathrm{nat}}(\theta)$ satisfying
$$
F(\theta)d_{\mathrm{nat}}(\theta)
=
\nabla_\theta \E_{\vect{y}\sim p_\theta(\cdot\mid\vect{x})}\big[w(\vect{x},\vect{y})\big]
$$
is a valid natural gradient direction (the solution is not unique because $F(\theta)\mathbf{1}=0$).
Therefore, one convenient and valid choice in the space $\mathbf{1}^{\top}d_{\mathrm{nat}}(\theta)=0$ is
\begin{equation}
\label{eq:natural gradient-choice}
d_{\mathrm{nat}}(\theta)=w(\vect{x},\cdot)
-\E_{\vect{y}\sim p_\theta(\cdot\mid\vect{x})}\big[w(\vect{x},\vect{y})\big]\mathbf{1}.
\end{equation}
Take $\theta^{+}=\theta+\eta d_{\mathrm{nat}}(\theta)$ with learning rate $\eta>0$.
A first-order Taylor expansion gives
$$
p_{\theta^{+}}(\vect{y}\mid\vect{x})-p_\theta(\vect{y}\mid\vect{x})
=
\eta\sum_{\vect{z}}\frac{\partial p_\theta(\vect{y}\mid\vect{x})}{\partial \theta_{\vect{z}}}
d_{\mathrm{nat}}(\theta)_{\vect{z}}
+O(\eta^2).
$$
Substituting the Jacobian formula and simplifying,
\begin{align*}
p_{\theta^{+}}(\vect{y}\mid\vect{x})-p_\theta(\vect{y}\mid\vect{x})
&=
\eta p_\theta(\vect{y}\mid\vect{x})
\Big(d_{\mathrm{nat}}(\theta)_{\vect{y}}-\sum_{\vect{z}}p_\theta(\vect{z}\mid\vect{x})d_{\mathrm{nat}}(\theta)_{\vect{z}}\Big)
+O(\eta^2) \\
&=
\eta p_\theta(\vect{y}\mid\vect{x})
\Big(d_{\mathrm{nat}}(\theta)_{\vect{y}}-\E_{\vect{z}\sim p_\theta(\cdot\mid\vect{x})}\big[d_{\mathrm{nat}}(\theta)_{\vect{z}}\big]\Big)
+O(\eta^2).
\end{align*}
By \Cref{eq:natural gradient-choice},
$$
\E_{\vect{z}\sim p_\theta(\cdot\mid\vect{x})}\big[d_{\mathrm{nat}}(\theta)_{\vect{z}}\big]
=
\E_{\vect{z}\sim p_\theta(\cdot\mid\vect{x})}\big[w(\vect{x},\vect{z})\big]
-
\E_{\vect{y}\sim p_\theta(\cdot\mid\vect{x})}\big[w(\vect{x},\vect{y})\big]
=
0,
$$
and hence
$$
p_{\theta^{+}}(\vect{y}\mid\vect{x})-p_\theta(\vect{y}\mid\vect{x})
=
\eta p_\theta(\vect{y}\mid\vect{x})
\Big(w(\vect{x},\vect{y})-\E_{\vect{y}\sim p_\theta(\cdot\mid\vect{x})}\big[w(\vect{x},\vect{y})\big]\Big)
+O(\eta^2).
$$
Therefore,
\begin{align*}
&\E_{\vect{y}\sim p_{\theta^{+}}(\cdot\mid\vect{x})}\big[r_m(\vect{x},\vect{y})\big]
-\E_{\vect{y}\sim p_{\theta}(\cdot\mid\vect{x})}\big[r_m(\vect{x},\vect{y})\big] \\
&\qquad=
\sum_{\vect{y}}\Big(p_{\theta^{+}}(\vect{y}\mid\vect{x})-p_\theta(\vect{y}\mid\vect{x})\Big)r_m(\vect{x},\vect{y}) \\
&\qquad=
\eta\sum_{\vect{y}} p_\theta(\vect{y}\mid\vect{x})
\Big(w(\vect{x},\vect{y})-\E_{\vect{y}\sim p_\theta(\cdot\mid\vect{x})}\big[w(\vect{x},\vect{y})\big]\Big)
r_m(\vect{x},\vect{y})
+O(\eta^2) \\
&\qquad=
\eta\Big(
\E_{\vect{y}\sim p_\theta(\cdot\mid\vect{x})}\big[r_m(\vect{x},\vect{y})w(\vect{x},\vect{y})\big]
-
\E_{\vect{y}\sim p_\theta(\cdot\mid\vect{x})}\big[r_m(\vect{x},\vect{y})\big]
\E_{\vect{y}\sim p_\theta(\cdot\mid\vect{x})}\big[w(\vect{x},\vect{y})\big]
\Big)
+O(\eta^2) \\
&\qquad=
\eta\mathrm{Cov}_{\vect{y}\sim p_{\theta}(\cdot\mid\vect{x})}\big(r_m(\vect{x},\vect{y}),w(\vect{x},\vect{y})\big)
+O(\eta^2),
\end{align*}
which proves the claim.
\end{proof}

\subsection{Proof of Corollary \ref{cor:clipping-robustness}}
\begin{proof}
Fix an objective $m$ and write $F\coloneq F(\theta)$, $R\coloneq R(\theta)$,
$G^{\mathrm{unclip}}\coloneq G^{\mathrm{unclip}}(\theta)$, and $G^{\mathrm{clip}}\coloneq G^{\mathrm{clip}}(\theta)$.
By definition,
$$
\gamma_m^{\mathrm{unclip}}(\theta)=\nabla r_m(\theta)^{\top}F^{-1}\big(G^{\mathrm{unclip}}-R\big).
$$
Hence we can decompose the clipped first-order margin as
\begin{align*}
\nabla r_m(\theta)^{\top}F^{-1}\big(G^{\mathrm{clip}}-R\big)
&=
\nabla r_m(\theta)^{\top}F^{-1}\big(G^{\mathrm{unclip}}-R\big)
-\nabla r_m(\theta)^{\top}F^{-1}\big(G^{\mathrm{unclip}}-G^{\mathrm{clip}}\big) \\
&=
\gamma_m^{\mathrm{unclip}}(\theta)
-\nabla r_m(\theta)^{\top}F^{-1}\big(G^{\mathrm{unclip}}-G^{\mathrm{clip}}\big).
\end{align*}
Now insert $F^{-1/2}$ and apply the Cauchy--Schwarz inequality:
\begin{align*}
\Big|\nabla r_m(\theta)^{\top}F^{-1}\big(G^{\mathrm{unclip}}-G^{\mathrm{clip}}\big)\Big|
&=
\Big|\big(F^{-1/2}\nabla r_m(\theta)\big)^{\top}\big(F^{-1/2}(G^{\mathrm{unclip}}-G^{\mathrm{clip}})\big)\Big| \\
&\le
\big\|F^{-1/2}\nabla r_m(\theta)\big\|\cdot
\big\|F^{-1/2}\big(G^{\mathrm{unclip}}-G^{\mathrm{clip}}\big)\big\|.
\end{align*}
Combining with the decomposition yields the margin bound
$$
\nabla r_m(\theta)^{\top}F^{-1}\big(G^{\mathrm{clip}}-R\big)
\ge
\gamma_m^{\mathrm{unclip}}(\theta)
-
\big\|F^{-1/2}\nabla r_m(\theta)\big\|\cdot
\big\|F^{-1/2}\big(G^{\mathrm{unclip}}-G^{\mathrm{clip}}\big)\big\|.
$$
Therefore, if $\gamma_m^{\mathrm{unclip}}(\theta)\ge \kappa_m>0$ and
$$
\big\|F^{-1/2}\big(G^{\mathrm{unclip}}-G^{\mathrm{clip}}\big)\big\|
\le
\frac{\kappa_m}{\big\|F^{-1/2}\nabla r_m(\theta)\big\|},
$$
then $\nabla r_m(\theta)^{\top}F^{-1}(G^{\mathrm{clip}}-R)\ge 0$, which proves the claim.
\end{proof}

\subsection{Proof of \Cref{theorem:pl for raw}}
\label{appendix:proof-pl}

We first restate, in full, the assumptions and the theorem.

\begin{assumption}[Bounded score and unique optimal completion]
\label{assumption:bounded-score}
There exists $B>0$ such that $|s(\vect{x},\vect{y})|\le B$ for all $\vect{y}\in\mathcal{X}^{L_{\mathrm{out}}}$.
Moreover, there exists a unique maximizer $\vect{y}^\ast = \arg\max_{\vect{y}} s(\vect{x},\vect{y})$ and a margin $\Delta_s>0$ such that
$s(\vect{x},\vect{y}^\ast)-s(\vect{x},\vect{y})\ge \Delta_s$ for all $\vect{y}\neq \vect{y}^\ast$.
\end{assumption}

\begin{assumption}[Non-saturation for suboptimal policy]
\label{assumption:ns}
There exists $\epsilon\in(0,1)$ such that for every suboptimal parameter $\theta$
(i.e., $V(\vect{x};\theta)<V(\vect{x};\theta^\ast)$),
token probabilities are bounded away from 1: $p_\theta\big(y_l \mid \vect{x}, \vect{y}_{\le l-1}\big)\le 1-\epsilon, \;\forall l\in\{1,\dots,L_{\mathrm{out}}\},\; \forall y_l\in\mathcal{X}$.
\end{assumption}

\begin{assumption}[Aligned token gradients]
\label{assumption:aligned-gradients}
There exists $c\in(0,1]$ such that for any $\vect{y}$ and any positions $l,k\in\{1,\dots,L_{\mathrm{out}}\}$ with
$v_l(\vect{x},\vect{y};\theta)\neq 0$ and $v_k(\vect{x},\vect{y};\theta)\neq 0$,
$$
\cos\big(v_l(\vect{x},\vect{y};\theta),v_k(\vect{x},\vect{y};\theta)\big)\ge c, \;\text{where } v_l(\vect{x},\vect{y};\theta)\coloneq J_{f(\vect{x},\vect{y}_{\le l-1};\theta)}^\top\Big(\vect{e}_{y_l}-p_\theta(\cdot\mid\vect{x},\vect{y}_{\le l-1})\Big)\in\R^n.
$$
\end{assumption}

$v_l(\vect{x},\vect{y};\theta)$ is the token-level logit-gradient contribution with $f$ the logit map and $J_f$ its Jacobian. These assumptions have clear interpretations in our setting. \Cref{assumption:bounded-score} is natural because every reward in the paper lies in $[0,1]$. \Cref{assumption:ns} matches practical RFT setups where KL regularization prevents token probabilities from saturating at $1$. \Cref{assumption:aligned-gradients} is stronger and model-dependent but has a direct geometric meaning: token-level gradient contributions along a favorable trajectory should be sufficiently aligned, so that they reinforce rather than cancel each other.

\begin{theorem}[$\mu$-PL condition of multi-objective alignment]
\label{theorem:pl for raw}
For parameters $\theta \in \R^n$ and input $\vect{x}\in\mathcal{X}^{L_\mathrm{in}}$, define the scalarized multi-objective value function $V(\vect{x};\theta)$ as in \Cref{eq:value function}. If the scalarization function $s$ and policy $\theta$ meet Assumptions~\ref{assumption:bounded-score}--\ref{assumption:aligned-gradients}, then it holds that:
\begin{align*}
\frac{1}{2}\|\nabla_\theta V(\vect{x};\theta)\|^2 \ge \mu(V(\vect{x};\theta^\ast) - V(\vect{x};\theta)),\;
\text{with } \mu = \frac{1}{2B}\Big(\frac{p_\theta(\vect{y}^\ast\mid\vect{x})}{1-p_\theta(\vect{y}^\ast\mid \vect{x})}s(\vect{x},\vect{y}^\ast)\gamma - 2B\sigma_\mathrm{max}\Big). \nonumber
\end{align*}
$\gamma = \sqrt{cL_\mathrm{out}\sigma_\mathrm{min}^2 \epsilon^2\frac{|\mathcal{X}^{L_\mathrm{out}}|}{|\mathcal{X}^{L_\mathrm{out}}|-1}}$. $\sigma_\mathrm{max}$ and $\sigma_\mathrm{min}$ are the largest and smallest singular values of $J_{f(\vect{x}, \vect{y}_{\leq l-1};\theta)}$.
\end{theorem}

We first recall the PL condition for maximization.
\begin{definition}
Let $V:\R^n \to \R$ be a differentiable objective function that we aim to \emph{maximize}. Define the set of global maximizers $\Theta^\ast \coloneq  \arg\max_{\theta \in \R^n} V(\theta)$, and assume $\Theta^\ast \neq \emptyset$. We say that $V$ satisfies the $\mu$-PL condition with parameter $\mu>0$ on a set $\Omega \subseteq \R^n$ if
\begin{equation}
\frac{1}{2}\big\|\nabla_\theta V(\theta)\big\|^2
\ge
\mu\big(V(\theta^\ast) - V(\theta)\big),\; \forall \theta \in \Omega,\; \theta^\ast \in \Theta^\ast.
\end{equation}
\end{definition}

In words, the PL condition ties gradient magnitude to global suboptimality, such that any point with a small gradient norm must have an objective value close to the global maximum, even though $V$ need not be convex. We now prove that the scalarized RFT objective satisfies this condition under Assumptions~\ref{assumption:bounded-score}--\ref{assumption:aligned-gradients}.

\begin{proof}
Differentiating $V(\vect{x};\theta)$ with respect to $\theta$ using the log-derivative trick gives:
\begin{align*}
\nabla_\theta V(\vect{x};\theta)
&= \E_{\vect{y}\sim p_\theta(\cdot\mid \vect{x})}
\Big[s(\vect{x},\vect{y})\nabla_\theta \ln p_\theta(\vect{y}\mid \vect{x})\Big]\\
&= \E_{\vect{y}\sim p_\theta(\cdot\mid \vect{x})}\Big[s(\vect{x},\vect{y})\sum_{l=1}^{L_\mathrm{out}}\nabla_\theta\ln p_\theta(y_l\mid \vect{x}, \vect{y}_{\leq l-1})\Big] \\
&= \E_{\vect{y}\sim p_\theta(\cdot\mid \vect{x})}\Big[s(\vect{x},\vect{y})\sum_{l=1}^{L_\mathrm{out}}\nabla_\theta\ln \mathrm{softmax}(f(\vect{x}, \vect{y}_{\leq l-1};\theta))_{y_l}\Big] \\
&= \sum_{\vect{y}\in \mathcal{X}^{L_\mathrm{out}}}p_\theta(\vect{y}\mid \vect{x})s(\vect{x},\vect{y})\sum_{l=1}^{L_\mathrm{out}}J^\top_{f(\vect{x}, \vect{y}_{\leq l-1};\theta)}(\vect{e}_{y_l} - p_\theta(\cdot \mid \vect{x},\vect{y}_{\leq l-1})),
\end{align*}
where $J_{f(\vect{x}, \vect{y}_{\leq l-1};\theta)}$ is the Jacobian of $f(\vect{x}, \vect{y}_{\leq l-1};\theta)$ with respect to $\theta$. $\vect{e}_{y_l}$ denotes the $y_l$'th standard basis vector whose $y_l$'th entry is 1 and all other entries are 0. $p_\theta(\cdot\mid\vect{x},\vect{y}_{\leq l-1}) = \mathrm{softmax}(f(\vect{x},\vect{y}_{\leq l-1};\theta)) \in \R^{|\mathcal{X}|}$ returns the next-token (at position $l$) probability distribution. Let $\phi(\vect{x},\vect{y};\theta) \coloneq \sum_{l=1}^{L_\mathrm{out}}J^\top_{f(\vect{x}, \vect{y}_{\leq l-1};\theta)}(\vect{e}_{y_l} - p_\theta(\cdot \mid \vect{x},\vect{y}_{\leq l-1}))$ and $\sigma_\mathrm{max}$ be the largest singular value of $J_{f(\vect{x}, \vect{y}_{\leq l-1};\theta)}$. By definition, we know $\sigma_\mathrm{max} = \max_{l\in \{1,2,\cdots,L_\mathrm{out}\}}\|J_{f(\vect{x},\vect{y}_{\leq l-1};\theta)}\|$. Obtaining the upper bound of $\phi(\vect{x},\vect{y};\theta)$ is nontrivial:
\begin{align*}
    \|\phi(\vect{x},\vect{y};\theta)\| \leq \|J^\top_{f(\vect{x}, \vect{y}_{\leq l-1};\theta)}\|_2 \cdot \|(\vect{e}_{y_l} - p_\theta(\cdot \mid \vect{x},\vect{y}_{\leq l-1}))\| \leq 2\sigma_\mathrm{max}
\end{align*}
because $\|(\vect{e}_{y_l} - p_\theta(\cdot \mid \vect{x},\vect{y}_{\leq l-1}))\| \leq \|(\vect{e}_{y_l} - p_\theta(\cdot \mid \vect{x},\vect{y}_{\leq l-1}))\|_1 \leq 2$ and $\|J^\top_{f(\vect{x}, \vect{y}_{\leq l-1};\theta)}\|_2 = \sigma_\mathrm{max}$ by definition. However, we cannot lower bound $\phi(\vect{x},\vect{y};\theta)$ without extra structure. So we first assume token probabilities are bounded away from 1 for all suboptimal $\theta$ in Assumption \ref{assumption:ns}, consistent with practical RFT setups where KL or entropy regularization prevents probabilities from overfitting to 1 during training. Assumption \ref{assumption:aligned-gradients} requires that all nonzero per-token contributions $v_l$ are positively aligned in parameter space and that their cosine similarity is uniformly bounded away from zero by $c$. With these assumptions about the structure of the policy function, we can derive the following lemmas that will be used in the proof later.
\begin{lemma}
Under Assumptions \ref{assumption:ns} and \ref{assumption:aligned-gradients}, we have 
\begin{align}
\label{eq:lower bound of trajectory gradient}
\Big\|\sum_{l=1}^{L_\mathrm{out}}J^\top_{f(\vect{x}, \vect{y}_{\leq l-1};\theta)}(\vect{e}_{y_l} - p_\theta(\cdot \mid \vect{x},\vect{y}_{\leq l-1}))\Big\|^2 \geq cL_\mathrm{out}\sigma_\mathrm{min}^2 \epsilon^2\frac{|\mathcal{X}^{L_\mathrm{out}}|}{|\mathcal{X}^{L_\mathrm{out}}|-1}.
\end{align}
\label{lemma:lower bound of trajectory gradient}
\end{lemma}
\begin{remark}
The left-hand side of \Cref{eq:lower bound of trajectory gradient} quantifies the squared norm of the summed token-level policy gradient contributions along an output sequence of length $L_{\mathrm{out}}$, and thus measures the overall strength of the gradient signal induced by that sequence. The lower bound shows that this signal scales with the alignment constant $c$, the smallest singular value $\sigma_{\min}$ of the logit Jacobian, the sequence length $L_{\mathrm{out}}$, and the probability gap $\epsilon$. Here, $\epsilon$ lower bounds $1 - p_\theta(y_l \mid \vect{x}, \vect{y}_{\le l-1})$ and captures the policy's uncertainty in predicting the next token. Smaller $\epsilon$ corresponds to higher confidence and more likely token sampling, but also leads to weaker gradient signals.
\end{remark}
\begin{proof}[Proof of Lemma \ref{lemma:lower bound of trajectory gradient}]
We lower bound the L2-norm of the vector $\vect{e}_{y_l} - p_\theta(\cdot\mid\vect{x},\vect{y}_{\leq l-1})$:
\begin{align}
    \label{eq:lower bound of prob dist}
    \|\vect{e}_{y_l} - p_\theta(\cdot\mid\vect{x},\vect{y}_{\leq l-1})\|^2 &= \big(1-p_\theta(y_l\mid \vect{x},\vect{y}_{\leq l-1})\big)^2 + \sum_{y^\prime \neq y_l,y^\prime\in\mathcal{X}}p_\theta^2(y^\prime\mid\vect{x},\vect{y}_{\leq l-1}) \nonumber \\ \nonumber
    &\geq \big(1-p_\theta(y_l\mid \vect{x},\vect{y}_{\leq l-1})\big)^2 + \frac{\big(\sum_{y^\prime \neq y_l,y^\prime\in\mathcal{X}}p_\theta(y^\prime\mid \vect{x},\vect{y}_{\leq l-1})\big)^2}{|\mathcal{X}^{L_\mathrm{out}}|-1} \\ \nonumber
    &=  \big(1-p_\theta(y_l\mid \vect{x},\vect{y}_{\leq l-1})\big)^2 + \frac{\big(1-p_\theta(y_l\mid\vect{x},\vect{y}_{\leq l-1})\big)^2}{|\mathcal{X}^{L_\mathrm{out}}|-1} \\ \nonumber
    &= \frac{|\mathcal{X}^{L_\mathrm{out}}|}{|\mathcal{X}^{L_\mathrm{out}}|-1}\big(1-p_\theta(y_l\mid\vect{x},\vect{y}_{\leq l-1})\big)^2 \\ 
    &\geq \epsilon^2\frac{|\mathcal{X}^{L_\mathrm{out}}|}{|\mathcal{X}^{L_\mathrm{out}}|-1}.
\end{align}
The first inequality is derived from the Cauchy--Schwarz inequality in Euclidean space $(\sum p_\theta)^2 \leq (\sum p_\theta^2)(\sum 1^2)$, and the last inequality is derived from Assumption~\ref{assumption:ns}. Then we have:
\begin{align*}
    \Big\|\sum_{l=1}^{L_\mathrm{out}}J^\top_{f(\vect{x}, \vect{y}_{\leq l-1};\theta)}(\vect{e}_{y_l} - p_\theta(\cdot \mid \vect{x},\vect{y}_{\leq l-1}))\Big\|^2 &= \sum_{l=1}^{L_\mathrm{out}}\Big\|v_l\Big\|^2 + 2\sum_{l<k}\langle v_l, v_k\rangle \\
    &\geq \sum_{l=1}^{L_\mathrm{out}}\Big\|v_l\Big\|^2 + 2c\sum_{l<k}\|v_l\|\|v_k\| \\
    & \geq c\Big(\sum_{l=1}^{L_\mathrm{out}}\|v_l\|\Big)^2 \\
    & \geq c\sigma_\mathrm{min}^2\Big(\sum_{l=1}^{L_\mathrm{out}}\|\vect{e}_{y_l} - p_\theta(\cdot\mid\vect{x},\vect{y}_{\leq l-1})\|\Big)^2 \\
    & \geq cL_\mathrm{out}\sigma_\mathrm{min}^2 \epsilon^2\frac{|\mathcal{X}^{L_\mathrm{out}}|}{|\mathcal{X}^{L_\mathrm{out}}|-1}.
\end{align*}
Here, $\sigma_{\min}$ denotes the smallest singular value of the Jacobian
$J_{f(\vect{x},\vect{y}_{\le l-1};\theta)}$.
The second inequality follows from Assumption~\ref{assumption:aligned-gradients} and factors out $c$. The third inequality applies the singular value bound $\|J_{f(\vect{x},\vect{y}_{\le l-1};\theta)}^\top u\| \ge \sigma_{\min}\|u\|,\;\forall u$. The final inequality uses the result from \Cref{eq:lower bound of prob dist}.
\end{proof}

\begin{lemma}
\label{lemma:bounded value function difference}
Let Assumption~\ref{assumption:bounded-score} hold and $\vect{y}^\ast$ is the unique maximizer of the scalarized reward, $\vect{y}^\ast \in \arg\max_{\vect{y}} s(\vect{x},\vect{y})$. Assume the policy class is rich enough to realize a deterministic policy on $\vect{x}$ that always outputs $\vect{y}^\ast$, i.e., there exists a parameter vector $\theta^\ast$ such that
\begin{equation*}
p_{\theta^\ast}(\vect{y}\mid \vect{x}) =
\begin{cases}
1, & \text{if } \vect{y} = \vect{y}^\ast,\\[3pt]
0, & \text{otherwise}.
\end{cases}
\end{equation*}
Then $\theta^\ast$ is an optimal parameter for the scalarized value at $\vect{x}$, and the optimal value becomes $V(\vect{x};\theta^\ast)
= \sum_{\vect{y}\in \mathcal{X}^{L_\mathrm{out}}} p_{\theta^\ast}(\vect{y}\mid \vect{x})s(\vect{x},\vect{y}) = s(\vect{x},\vect{y}^\ast)$.
For any $\theta$, we have: 
\begin{align*}
V(\vect{x};\theta^\ast) - V(\vect{x};\theta) \leq 2B\big(1 - p_\theta(\vect{y}^\ast\mid \vect{x})\big).
\end{align*}
\end{lemma}
\begin{remark}
Lemma \ref{lemma:bounded value function difference} shows that the value gap is Lipschitz in the probability gap on $\vect{y}^\ast$. Every unit of probability that fails to go to the optimal sequence can hurt the value by at most $2B$. Equivalently, driving the model toward near-deterministic predictions on $\vect{y}^\ast$, $p_\theta(\vect{y}^\ast\mid \vect{x}) \approx 1$, is both necessary and sufficient (up to the factor $2B$) to make the scalarized value nearly optimal. This connects convergence in value directly to how sharply the policy concentrates on the optimal sequence.
\end{remark}

\begin{proof}[Proof of Lemma \ref{lemma:bounded value function difference}]
We can write $V(\vect{x};\theta^\ast) - V(\vect{x};\theta)$ as follows:
\begin{align*}
    V(\vect{x};\theta^\ast) - V(\vect{x};\theta) &= s(\vect{x},\vect{y}^\ast) - \sum_{\vect{y}\in\mathcal{X}^{L_\mathrm{out}}}p_\theta(\vect{y}\mid \vect{x})s(\vect{x},\vect{y})  \\
    &= s(\vect{x},\vect{y}^\ast) - \big(p_\theta(\vect{y}^\ast\mid\vect{x})s(\vect{x},\vect{y}^\ast) + \sum_{\vect{y}\neq \vect{y}^\ast}p_\theta(\vect{y}\mid\vect{x})s(\vect{x},\vect{y})\big) \\
    &= \big(1-p_\theta(\vect{y}^\ast\mid \vect{x})\big)s(\vect{x},\vect{y}^\ast) - \sum_{\vect{y}\neq \vect{y}^\ast}p_\theta(\vect{y}\mid\vect{x})s(\vect{x},\vect{y}) \\
    & \leq \big(1-p_\theta(\vect{y}^\ast\mid \vect{x})\big)s(\vect{x},\vect{y}^\ast) - \min_{\vect{y}\neq \vect{y}^\ast} s(\vect{x},\vect{y}) \sum_{\vect{y}\neq \vect{y}^\ast}p_\theta(\vect{y}\mid\vect{x}) \\
    & = \big(1-p_\theta(\vect{y}^\ast\mid \vect{x})\big)s(\vect{x},\vect{y}^\ast) - \min_{\vect{y}\neq \vect{y}^\ast} s(\vect{x},\vect{y})\big(1-p_\theta(\vect{y}^\ast\mid \vect{x})\big) \\
    & = \big(s(\vect{x},\vect{y}^\ast)-\min_{\vect{y}\neq \vect{y}^\ast} s(\vect{x},\vect{y})\big)\big(1-p_\theta(\vect{y}^\ast\mid\vect{x})\big) \\
    & \leq 2B\big(1 - p_\theta(\vect{y}^\ast\mid \vect{x})\big).
\end{align*}
The last inequality follows from Assumption~\ref{assumption:bounded-score} that $|s(\vect{x},\vect{y})|\leq B$ and $\|s(\vect{x},\vect{y}^\ast) - \min_{\vect{y}\neq \vect{y}^\ast} s(\vect{x},\vect{y})\| \leq \|s(\vect{x},\vect{y}^\ast) - \min_{\vect{y}\neq \vect{y}^\ast} s(\vect{x},\vect{y})\|_1 \leq 2B$. 
\end{proof}

We define a direction vector:
$$
u(\vect{x},\vect{y}^\ast;\theta) \coloneq \frac{\phi(\vect{x},\vect{y}^\ast;\theta)}{\|\phi(\vect{x},\vect{y}^\ast;\theta)\|} \coloneq \frac{\sum_{l=1}^{L_\mathrm{out}}J_{f(\vect{x},\vect{y}_{\le l-1};\theta)}^\top
\big(\vect{e}_{y_l} - p_\theta(\cdot \mid \vect{x},\vect{y}_{\le l-1})\big)}{\big\|\sum_{l=1}^{L_\mathrm{out}}J_{f(\vect{x},\vect{y}_{\le l-1};\theta)}^\top
\big(\vect{e}_{y_l} - p_\theta(\cdot \mid \vect{x},\vect{y}_{\le l-1})\big)\big\|},
$$ 
which is the normalized log-probability gradient of the optimal trajectory. We first discuss its inner product with $\phi(\vect{x},\vect{y};\theta)$ which will be used for the proof later.
\begin{equation}
\langle \phi(\vect{x},\vect{y};\theta),u(\vect{x},\vect{y}^\ast;\theta)\rangle
\begin{cases}
=\|\phi(\vect{x},\vect{y}^\ast;\theta)\| \geq \sqrt{cL_\mathrm{out}\sigma_\mathrm{min}^2 \epsilon^2\frac{|\mathcal{X}^{L_\mathrm{out}}|}{|\mathcal{X}^{L_\mathrm{out}}|-1}}, & \text{if } \vect{y} = \vect{y}^\ast,\\[3pt]
\leq \|\phi(\vect{x},\vect{y};\theta)\|\leq 2\sigma_\mathrm{max}, & \forall \vect{y}.
\end{cases}
\label{eq:dot product with phi}
\end{equation}
Then the directional derivative of $V(\vect{x};\theta)$ along $u(\vect{x},\vect{y}^\ast;\theta)$ is:
\begin{align*}
    \langle \nabla_\theta V(\vect{x};\theta),u(\vect{x},\vect{y}^\ast;\theta) \rangle = \E_{\vect{y}\sim p_\theta(\cdot\mid \vect{x})}\Big[s(\vect{x},\vect{y})\langle \phi(\vect{x},\vect{y};\theta),u(\vect{x},\vect{y}^\ast;\theta)\rangle \Big] .
\end{align*}
Extracting $\vect{y}^\ast$ out yields:
\begin{align}
\langle \nabla_\theta V(\vect{x};\theta),u(\vect{x},\vect{y}^\ast;\theta) \rangle = \underbrace{p_\theta(\vect{y}^\ast\mid\vect{x})s(\vect{x},\vect{y}^\ast)\langle \phi(\vect{x},\vect{y}^\ast;\theta),u(\vect{x},\vect{y}^\ast;\theta)\rangle}_{(I)} + \nonumber \\   \underbrace{\sum_{\vect{y}\neq \vect{y}^\ast}p_\theta(\vect{y}\mid\vect{x})s(\vect{x},\vect{y})\langle \phi(\vect{x},\vect{y};\theta),u(\vect{x},\vect{y}^\ast;\theta)\rangle}_{(II)}.
\label{eq:directional derivative}
\end{align}
We lower bound $(I)$ and $(II)$ separately. Starting with $(I)$, by \Cref{eq:dot product with phi}, we know that:
\begin{align*}
    (I) \geq  p_\theta(\vect{y}^\ast\mid\vect{x})s(\vect{x},\vect{y}^\ast)\sqrt{cL_\mathrm{out}\sigma_\mathrm{min}^2 \epsilon^2\frac{|\mathcal{X}^{L_\mathrm{out}}|}{|\mathcal{X}^{L_\mathrm{out}}|-1}}.
\end{align*}
Recalling that $|s(\vect{x},\vect{y})| \leq B$ from Assumption \ref{assumption:bounded-score} and the property $\langle \phi(\vect{x},\vect{y};\theta),u(\vect{x},\vect{y}^\ast;\theta)\rangle \leq 2\sigma_\mathrm{max}$, we get:
\begin{align*}
    (II) \geq -2B\sigma_\mathrm{max}\sum_{\vect{y}\neq \vect{y}^\ast}p_\theta(\vect{y}\mid\vect{x}) = -2B\sigma_\mathrm{max}(1-p_\theta(\vect{y}^\ast\mid \vect{x})).
\end{align*}
For $(II)$, the worst case happens when, for every $\vect{y}\neq \vect{y}^\ast$, the term's contribution to the directional derivative along $u$ is as adverse as allowed by the assumed bounds. Plugging the above two inequalities into \Cref{eq:directional derivative} and rearranging, we obtain:
\begin{align}
    \langle \nabla_\theta V(\vect{x};\theta),u(\vect{x},\vect{y}^\ast;\theta) \rangle \geq s(\vect{x},\vect{y}^\ast) \gamma - \big(s(\vect{x},\vect{y}^\ast)\gamma  + 2B\sigma_\mathrm{max}\big)\big(1-p_\theta(\vect{y}^\ast\mid \vect{x})\big)&, \nonumber \\ 
    \text{with } \gamma = \sqrt{cL_\mathrm{out}\sigma_\mathrm{min}^2 \epsilon^2\frac{|\mathcal{X}^{L_\mathrm{out}}|}{|\mathcal{X}^{L_\mathrm{out}}|-1}}&.
    \label{eq:directional derivative bound}
\end{align}
The constant $\mu$ quantifies how ``sharp'' the scalarized landscape $V(\vect{x};\theta)$ is around its maximizer $\theta^\ast$. It is large when the policy places large mass on the optimal completion $\vect{y}^\ast$, when the optimal score $s(\vect{x},\vect{y}^\ast)$ is large, and when the logit map is well-conditioned (large $\gamma$, which summarizes the Jacobian and non-saturation constants). Conversely, $\mu$ becomes non-positive when the aligned ascent signal from the optimal completion is weak (small $p_\theta(\vect{y}^\ast\mid \vect{x})$ or small $\gamma$) or when the logit map is highly skewed (large $\sigma_{\max}/\sigma_{\min}$), so the identical scalarized update can converge on one model and stall on another.

Intuitively, this inequality shows that when the policy is close to the optimal $\theta^\ast$ and $\vect{y}^\ast$ is the (unique) optimal trajectory maximizing the scalarized reward, the directional derivative along $u$ admits a large lower bound and thus yields a strong, well-aligned gradient update that pushes additional probability mass toward $\vect{y}^\ast$. 

Recall that we have upper bounded the value function gap by the sequence probability gap in Lemma \ref{lemma:bounded value function difference}, where we assume $p_{\theta^\ast}(\vect{y}^\ast\mid\vect{x})=1$. We are close to the result and want a final inequality that looks like:
\begin{equation}
\|\nabla_\theta V(\vect{x};\theta)\| \geq \langle \nabla_\theta V(\vect{x};\theta), u(\vect{x},
\vect{y}^\ast;\theta)\rangle \geq \mu^\prime\big(p_{\theta^\ast}(\vect{y}^\ast\mid\vect{x}) - p_\theta(\vect{y}^\ast\mid\vect{x})\big) \geq \mu\big(V(\vect{x};\theta^\ast) - V(\vect{x};\theta)\big).
\label{eq:final inequality}
\end{equation}
The first inequality follows from the Cauchy--Schwarz inequality applied to the unit vector and the last inequality comes from Lemma \ref{lemma:bounded value function difference}. So the only task left here is to get a $\mu^\prime$ that satisfies the second inequality for all $\theta$. Combining with \Cref{eq:directional derivative bound}, we have:
\begin{align*}
    \mu^\prime \leq \frac{p_\theta(\vect{y}^\ast\mid\vect{x})}{1-p_\theta(\vect{y}^\ast\mid \vect{x})}s(\vect{x},\vect{y}^\ast)\gamma - 2B\sigma_\mathrm{max}.
\end{align*}
Because the function $\frac{p}{1-p}$ is increasing on $[0,1)$, a sufficient and tight choice for $\mu^\prime$ is:
$$
\mu^\prime \coloneq \frac{p_\theta(\vect{y}^\ast\mid\vect{x})}{1-p_\theta(\vect{y}^\ast\mid \vect{x})}s(\vect{x},\vect{y}^\ast)\gamma - 2B\sigma_\mathrm{max}.
$$
Therefore, the final $\mu$-PL condition follows from \Cref{eq:final inequality}, and combining it with the bound from Lemma \ref{lemma:bounded value function difference} leads to:
\begin{align}
    \mu = \frac{1}{2B}\Big(\frac{p_\theta(\vect{y}^\ast\mid\vect{x})}{1-p_\theta(\vect{y}^\ast\mid \vect{x})}s(\vect{x},\vect{y}^\ast)\gamma - 2B\sigma_\mathrm{max}\Big).
\end{align}
\end{proof}

\clearpage
\section{Algorithms}
\label{appendix:algorithms}
We provide pseudocode for algorithms that, to our knowledge, have not been previously applied to LLM alignment. Throughout the paper we report the smooth variant of Tchebycheff scalarization rather than the classical one \citep{10.1007/978-3-642-87563-2_5}, because it is more stable and stronger in our experiments. For Dynamic \citep{lu2026learningoptimizemultiobjectivealignment}, PAMA \citep{10.1007/978-3-032-06078-5_15}, and Linear, we follow the original implementations and omit the pseudocode.

\begin{algorithm}[ht]
\caption{\method{} on top of GRPO}
\label{alg:ctwa}
\begin{algorithmic}[1]
\STATE \textbf{Inputs:} Objectives $\{r_m\}_{m=1}^M$, scalarization weights $\{\lambda_m\}$, covariance targets $\{c_m^\ast\}$, EMA rate $\tau\in(0,1]$, weight learning rate $\eta_\lambda$.
\STATE Initialize $u_m \leftarrow \log \lambda_m$.
\FOR{each training iteration $t=1,2,\ldots$}
  \STATE Set reference policy $\theta_{\mathrm{ref}} \leftarrow \theta$.
  \STATE Sample a batch of prompts $\{\vect{x}_i\}_{i=1}^B$.
  \FOR{each prompt $\vect{x}_i$}
    \STATE Sample $K$ completions $\{\vect{y}_i^{(k)}\}_{k=1}^K \sim p_{\theta_{\mathrm{ref}}}(\cdot\mid\vect{x}_i)$.
    \STATE Evaluate objective rewards $r_m(\vect{x}_i,\vect{y}_i^{(k)})$ and scalar scores $s_{\vect{\lambda}}(\vect{x}_i,\vect{y}_i^{(k)})$.
    \STATE Compute group-normalized GRPO advantages $\{A^b_{i,k}\}_{k=1}^K$ from $\{s_{\vect{\lambda}}(\vect{x}_i,\vect{y}_i^{(k)})\}_{k=1}^K$.
  \ENDFOR
  \STATE \textbf{Inner update (GRPO).} Update $\theta$ by optimizing the standard clipped GRPO surrogate using $\theta_{\mathrm{ref}}$ and $\{A^b_{i,k}\}$.
  \STATE \textbf{\method{} statistics.} Using the same batch, compute completion-level clipped advantage weights $w(\vect{x}_i,\vect{y}_i^{(k)};\theta)$.
  \STATE Compute per-objective covariance $\{c_m\}_{m=1}^M$ by aggregating within each prompt group and averaging across the batch.
  \STATE Update EMA: $\bar{c}_m \leftarrow (1-\tau)\bar{c}_m + \tau c_m$ for all $m$.
  \STATE Compute deficits $\delta_m \leftarrow [\,c_m^\ast-\bar{c}_m\,]_+$.
  \STATE \textbf{Outer update (\method{}).} Update log-weights $u_m \leftarrow u_m + \eta_\lambda\,\delta_m$ and set $\lambda_m \leftarrow \exp(u_m)$ for all $m$.
\ENDFOR
\end{algorithmic}
\end{algorithm}

\begin{algorithm}[ht]
\caption{GradNorm on top of GRPO}
\label{alg:gradnorm}
\begin{algorithmic}[1]
\STATE \textbf{Inputs:} Objectives $\{r_m\}_{m=1}^M$, GradNorm exponent $\alpha$, weight learning rate $\eta_w$.
\STATE Initialize weights $w_m \leftarrow 1$ for all $m$; initialize reference losses $L_m^{(0)} \leftarrow \texttt{unset}$.
\FOR{each training iteration $t=1,2,\ldots$}
  \STATE \textbf{Same as \method{}.} Sample prompts and $K$ completions; compute rewards and GRPO quantities needed to form a clipped surrogate.
  \STATE \textbf{Per-objective gradients.} Compute each objective's GRPO surrogate loss value $L_m$ and its policy gradient $g_m$ (using the same KL/entropy regularization as the base run).
  \IF{$L_m^{(0)}$ is unset}
    \STATE Set $L_m^{(0)} \leftarrow L_m$ for all $m$.
  \ENDIF
  \STATE \textbf{GradNorm targets.} Compute gradient norms $G_m \leftarrow \|g_m\|_2$ and each objective's relative training rate from the loss ratio $L_m / L_m^{(0)}$; convert it to a target scaled norm using exponent $\alpha$.
  \STATE \textbf{Weight update.} Update $w_m$ with step size $\eta_w$ to match the scaled norms $w_m G_m$ to their targets; renormalize weights.
  \STATE \textbf{Policy update.} Apply one optimization step using the weighted gradient $\sum_{m=1}^M w_m g_m$.
\ENDFOR
\end{algorithmic}
\end{algorithm}

\begin{algorithm}[ht]
\caption{Nash-MTL on top of GRPO}
\label{alg:nashmtl}
\begin{algorithmic}[1]
\STATE \textbf{Inputs:} Objectives $\{r_m\}_{m=1}^M$, Nash-MTL optimization iterations $T_{\mathrm{Nash}}$, Nash learning rate $\eta_\alpha$, numerical constant $\epsilon$, maximum weight $\alpha_{\max}$.
\STATE Initialize Nash bargaining weights $\alpha_m \leftarrow 1$ for all $m$.
\FOR{each training iteration $t=1,2,\ldots$}
  \STATE \textbf{Same as \method{}.} Sample prompts and $K$ completions; compute per-objective rewards, per-objective advantages, and GRPO quantities needed to form a clipped surrogate.
  \STATE \textbf{Per-objective gradients.} For each objective $m$, compute its GRPO surrogate loss value $L_m$ and policy gradient $g_m$ using the same KL/entropy regularization as the base run.
  \STATE \textbf{Gradient matrix.} Form the per-objective gradient matrix $J \in \mathbb{R}^{M \times P}$ by stacking flattened gradients, where the $m$-th row is $g_m^\top$.
  \STATE \textbf{Nash bargaining weights.} Compute the Gram matrix $G \leftarrow JJ^\top$ and approximately solve the Nash-MTL bargaining problem for $\boldsymbol{\alpha}\in\mathbb{R}_+^M$, warm-started from the previous $\boldsymbol{\alpha}$, such that $      \alpha_m (G\boldsymbol{\alpha})_m \approx 1$. Clamp $\boldsymbol{\alpha}$ to $[\epsilon,\alpha_{\max}]$ for numerical stability.
  \STATE \textbf{Policy update.} Apply one optimization step using the Nash-weighted gradient $\sum_{m=1}^M \alpha_m g_m$.
\ENDFOR
\end{algorithmic}
\end{algorithm}

\begin{algorithm}[ht]
\caption{FAMO on top of GRPO}
\label{alg:famo}
\begin{algorithmic}[1]
\STATE \textbf{Inputs:} Objectives $\{r_m\}_{m=1}^M$, weight learning rate $\eta_w$, weight decay $\gamma$, numerical constant $\epsilon$, loss margin $\delta$.
\STATE Initialize weight logits $\mathbf{w}\leftarrow \mathbf{0}$ and lower bounds $L_m^{\min}\leftarrow \texttt{unset}$.
\FOR{each training iteration $t=1,2,\ldots$}
  \STATE \textbf{Same as \method{}.} Sample prompts and $K$ completions; compute per-objective rewards, advantages, and GRPO quantities needed to form a clipped surrogate.
  \STATE \textbf{Per-objective losses.} Compute each objective's GRPO surrogate loss $L_m$ using the same KL/entropy regularization as the base run.
  \IF{$L_m^{\min}$ is unset}
    \STATE Set $L_m^{\min}\leftarrow L_m-\delta$ for all $m$.
  \ENDIF
  \STATE \textbf{FAMO weighted loss.} Let $\mathbf{z}\leftarrow\mathrm{softmax}(\mathbf{w})$, $d_m\leftarrow L_m-L_m^{\min}+\epsilon$, and $c\leftarrow\sum_m z_m/d_m$; form $L_{\mathrm{FAMO}}\leftarrow\sum_m (z_m/c)\log d_m$.
  \STATE \textbf{Policy update.} Apply one optimization step using $\nabla_\theta L_{\mathrm{FAMO}}$.
  \STATE \textbf{Weight update.} Recompute post-update losses $\widetilde{L}_m$ and set $\Delta_m\leftarrow \log(L_m-L_m^{\min}+\epsilon)-\log(\widetilde{L}_m-L_m^{\min}+\epsilon)$.
  \STATE Update $\mathbf{w}$ with Adam using the gradient induced by $\boldsymbol{\Delta}$ through $\mathbf{z}=\mathrm{softmax}(\mathbf{w})$.
\ENDFOR
\end{algorithmic}
\end{algorithm}

\newpage
\begin{algorithm}[ht]
\caption{MGDA on top of GRPO}
\label{alg:mgda}
\begin{algorithmic}[1]
\STATE \textbf{Inputs:} Objectives $\{r_m\}_{m=1}^M$.
\FOR{each training iteration $t=1,2,\ldots$}
  \STATE \textbf{Same as \method{}.} Sample prompts and $K$ completions; compute rewards and GRPO quantities needed to form a clipped surrogate.
  \STATE \textbf{Per-objective gradients.} Compute each objective's GRPO surrogate policy gradient $g_m$ (using the same KL/entropy regularization as the base run).
  \STATE \textbf{MGDA weights.} Solve for simplex weights $w\in\Delta^{M}$ that minimize the norm of the combined gradient, i.e., find the minimum-norm convex combination of $\{g_m\}_{m=1}^M$.
  \STATE \textbf{Policy update.} Apply one optimization step using the weighted gradient $\sum_{m=1}^M w_m g_m$.
\ENDFOR
\end{algorithmic}
\end{algorithm}

% \begin{algorithm}[ht]
% \caption{Weighted Tchebycheff Scalarization on top of GRPO}
% \label{alg:tchebycheff}
% \begin{algorithmic}[1]
% \STATE \textbf{Inputs:} Objectives $\{r_m\}_{m=1}^M$, scalarization weights $\{w_m\}$.
% \STATE Initialize running reference point $z\in\mathbb{R}^M$ using the first batch.
% \FOR{each training iteration $t=1,2,\ldots$}
%   \STATE \textbf{Same as \method{}.} Sample prompts and $K$ completions; compute per-objective token-level rewards $r_m(\vect{x}_i,\vect{y}_i^{(k)})$ and GRPO quantities needed to form a clipped surrogate.
%   \STATE \textbf{Update reference point.} Update $z_m \leftarrow \max\{z_m,\; \max_{\text{batch,tokens}} r_m\}$ for all $m$.
%   \STATE \textbf{Scalar score.} For each token, compute the weighted Tchebycheff scalar reward $s(\cdot) \leftarrow - \max_{m} \; w_m \bigl(z_m - r_m(\cdot)\bigr)$.
%   \STATE \textbf{Policy update.} Compute GRPO advantages from the scalar score $s$ and apply one optimization step.
% \ENDFOR
% \end{algorithmic}
% \end{algorithm}

\begin{algorithm}[ht]
\caption{Smooth Tchebycheff Scalarization with Online Mirror Descent on top of GRPO}
\label{alg:smooth-tchebycheff}
\begin{algorithmic}[1]
\STATE \textbf{Inputs:} Objectives $\{r_m\}_{m=1}^M$, preference weights $\{\lambda_m\}_{m=1}^M$, OMD learning rate $\alpha$.
\FOR{each training iteration $t=1,2,\ldots$}
    \STATE \textbf{Same as \method{}.} Sample prompts and $K$ completions, compute rewards and GRPO quantities needed to form a clipped surrogate.
    \STATE \textbf{Estimate objective values.} For each objective, sum token-level rewards into sequence-level returns and compute the batch mean objective value $R_m$.
    \STATE \textbf{Update the reference point.} For each objective, update $z_m$ using the running maximum sequence-level return observed so far.
    \STATE \textbf{Compute smooth weights.} Convert the current log-weights into smooth indicator weights by $w \leftarrow \text{softmax}(u)$.
    \STATE \textbf{Scalarize rewards.} For each token, compute the scalar reward $s(\cdot) \leftarrow \sum_{m=1}^M w_m r_m(\cdot)$.
    \STATE \textbf{Update log-weights by OMD.} For each objective, update $u_m \leftarrow u_m + \alpha \lambda_m (z_m - R_m)$.
    \STATE \textbf{Policy update.} Compute advantages from the scalar reward $s$ and apply one optimization step.
\ENDFOR
\end{algorithmic}
\end{algorithm}

\begin{algorithm}[H]
\caption{Lagrangian Primal-Dual Method on top of GRPO}
\label{alg:lagrangian}
\begin{algorithmic}[1]
\STATE \textbf{Inputs:} Primary objective $r_0$, constraint objectives $\{r_k\}_{k=1}^K$, target constraint rewards $\{c_k\}$, dual learning rate $\eta_\lambda$.
\STATE Initialize Lagrange multipliers $\lambda_k \leftarrow 0$ for all $k$.
\FOR{each training iteration $t=1,2,\ldots$}
  \STATE \textbf{Same as \method{}.} Sample prompts and $K$ completions; compute per-objective token-level rewards $\{r_0, r_1,\ldots,r_K\}$ and GRPO quantities needed to form a clipped surrogate.
  \STATE \textbf{Per-objective advantages.} Compute advantages for each objective separately, yielding $A_0$ (primary) and $\{A_k\}_{k=1}^K$ (constraints).
  \STATE \textbf{Dual update.} Estimate each constraint objective's average reward on the current batch, compute the reward gaps $c_k - \mathbb{E}[r_k]$, and update multipliers by dual ascent: $\lambda_k \leftarrow \max\{0,\; \lambda_k + \eta_\lambda\,(c_k - \mathbb{E}[r_k])\}$ for all $k$.
  \STATE \textbf{Primal scalarization.} Form a Lagrangian-weighted advantage: $A \leftarrow A_0 + \sum_{k=1}^K \lambda_k A_k$.
  \STATE \textbf{Policy update.} Apply one optimization step using the combined advantage $A$.
\ENDFOR
\end{algorithmic}
\end{algorithm}

\clearpage
\section{Hyperparameters}
\begin{table*}[ht]
\centering
\small
\setlength{\tabcolsep}{4pt}
\renewcommand{\arraystretch}{1.12}
\caption{Hyperparameter summary for different scalarization algorithms. Unless noted, all runs use the shared policy update and compute settings listed below.}
\label{table:hyperparameters}
\begin{tabular}{@{}p{0.19\textwidth}p{0.76\textwidth}@{}}
\toprule
Method & Hyperparameters \\
\midrule
Lagrangian primal-dual &
On MATH-500 the primary objective is accuracy,
with constraints on conciseness and clarity (targets $[0.9,0.9]$);
on IFEval the primary objective is constraint following,
with constraints on helpfulness ($\ge0.70$) and conciseness ($\ge0.85$);
dual learning rate $\eta_\lambda=0.01$;
KL coefficient $=0$. \\

Smooth Tchebycheff &
Preference weights $\lambda_m=[0.333,0.333,0.334]$;
OMD learning rate $\alpha=0.3$. \\

MGDA &
Same shared settings, except KL coefficient $=0$. \\

GradNorm &
Exponent $\alpha=1.5$;
weight learning rate $\eta_w=0.025$. \\

Nash-MTL &
Nash optimization iterations $T_{\mathrm{Nash}}=20$;
weight learning rate $\eta_\alpha=0.1$;
numerical constant $\epsilon=10^{-8}$;
maximum weight $\alpha_{\max}=1000$. \\

FAMO &
Weight learning rate $\eta_w=0.025$;
weight decay $\gamma=10^{-5}$;
numerical constant $\epsilon=10^{-8}$;
loss margin $\delta=1.0$. \\

\method{} &
Initial weights $\lambda^{(0)}=[0.333,0.333,0.334]$;
weight learning rate $\eta_\lambda=0.05$;
covariance targets $c^\ast=[0.15,0.05,0.05]$;
EMA rate $\tau=0.1$. \\

\midrule
Shared policy update &
Learning rate $\eta=10^{-6}$;
LR scheduler is constant;
batch size $=32$;
rollouts $K=16$ per prompt;
sampling temperature $=1.0$;
max prompt length $=1024$;
max response length $=1024$;
entropy coefficient $=0$;
KL coefficient $=0.001$;
clipping coefficient $\epsilon=0.2$ for GRPO;
epochs $=90$;
optimization backend is FSDP. \\

Shared compute &
GPUs are $4\times$ Nvidia L40 48\,GB;
inference engine is vLLM. \\
\bottomrule
\end{tabular}
\end{table*}

\newpage

\end{document}